\documentclass[journal]{IEEEtran}
\IEEEoverridecommandlockouts                              % This command is only needed if 
\usepackage{setspace}
\usepackage{graphics} % for pdf, bitmapped graphics files
\usepackage{epsfig} % for postscript graphics files
\usepackage{times} % assumes new font selection scheme installed
\usepackage{amsmath} % assumes amsmath package installed
\usepackage{amssymb}  % assumes amsmath package installed
\usepackage{subfigure}
\usepackage{epstopdf}
\usepackage{cite}
\usepackage[final]{changes}
\usepackage{xcolor}
\usepackage{bbm}
\usepackage{tikz}
\usepackage{caption}
\usepackage{algpseudocode}
\usepackage{mathtools}
\usepackage{amsthm}

\makeatletter
\newcommand{\removelatexerror}{\let\@latex@error\@gobble}
\makeatother
\usetikzlibrary{automata,positioning,shapes,arrows}

\newtheorem{theorem}{Theorem}
\newtheorem{lemma}{Lemma}
\usepackage[lined,ruled,commentsnumbered]{algorithm2e}

\newtheorem{problem}{Problem}

\newtheorem{definition}{Definition}
\newtheorem{assumption}{Assumption}
\newtheorem{remark}{Remark}

\begin{document}
	
	\title{\Large \bf A Learning Based Optimal Human Robot Collaboration with Linear Temporal Logic Constraints}
	
	\author{Bo Wu, Bin Hu and Hai Lin	
		\thanks{This work was supported by the National Science Foundation (NSF-CNS-1446288  and NSF- ECCS-1253488). }	
		\thanks{
			Bin Hu is with New Mexico State University, Las Cruces, NM, 88003, USA. {\tt\small binhu.complicated@gmail.com}} 
		\thanks{Bo Wu and Hai Lin are with University of Notre Dame, Notre Dame, IN, 46556, USA. {\tt\small bwu3@nd.edu, hlin1@nd.edu}}}
	\maketitle
	
	\begin{abstract}
		This paper considers an optimal task allocation problem  for human robot collaboration in human robot systems with persistent tasks. Such human robot systems consist of human operators and intelligent robots collaborating with each other to accomplish complex tasks that cannot be done by either part alone. The system objective is to maximize the probability of successfully executing
		persistent tasks that are formulated as linear temporal logic specifications and minimize the average cost between consecutive visits of a particular proposition. This paper proposes to model the human robot collaboration under a framework with the composition of multiple Markov Decision Process (MDP) with possibly unknown transition probabilities, which characterizes how human cognitive states, such as human trust and fatigue, stochastically change with the robot performance. Under the unknown MDP models, an algorithm is developed to learn the model and obtain an optimal task allocation policy that minimizes the expected average cost for each task cycle and maximizes the probability of satisfying linear temporal logic constraints. Moreover, this paper shows that the difference between the optimal policy based on the learned model and that based on the underlying ground truth model can be bounded by arbitrarily small  constant and large confidence level with sufficient samples. The case study of an assembly process demonstrates the effectiveness and benefits of our proposed learning based human robot collaboration.
		%In practice, the transition probabilities in such MDP models may be unknown a prior. To this regard, we develop an effective learning algorithm to efficiently estimate the model with performance guarantee. Then by formulating the optimal task allocation problem as a controller synthesis problem subject to minimizing the expected average cost for each task cycle and maximizing the probability of satisfying linear temporal logic constraints, we prove that the difference between the optimal policy based on the learned model and that based on the underlying ground truth model is bounded with a predefined small constant and confidence level. The case study of an assembly process demonstrates the effectiveness and benefits of our proposed learning based framework.
		
	\end{abstract}
	\renewcommand{\abstractname}{Note to Practitioners}
 
		\begin{abstract}
This paper is motivated by the task allocation problem for a class of human robot collaboration systems. In such systems, the most efficient way to achieve a given task is for the human and the robot to smartly collaborate with each other. Our objective is to dynamically assign tasks to the human and the robot based on human cognitive and physiologic states as well as the robot's performance, so that the given task can be accomplished with the optimal probability and the minimized average cost to finish one round of the operation. To this end, we propose a mathematical model to capture the dynamic evolution of human cognitive and physiologic states, such as trust and fatigue. Such model in practice could be unknown initially, so we develop an efficient learning based algorithm to learn the model with predefined accuracy. Then we show that the performance of the optimal policy obtained on the learned model is sufficiently close to the optimal policy on the underlying ground truth model. Our proposed methods are applicable to manufacturing assembly process, semi-autonomous driving and other human robot collaboration systems that require optimal allocation of the tasks between human and robot. The limitation of the proposed approach is the assumption of the perfect real time observation of the human states. The authors believe that such limitation can be addressed with the recent advances in the cognitive science research. 
			
		\end{abstract}
		
	%	\begin{IEEEkeywords}
	%		Factory Automation System, Stochastic Safety, Shadow Fading, Optimal Co-design, Generalized Geometric Programming
	%	\end{IEEEkeywords}
	\IEEEpeerreviewmaketitle
	
\section{Introduction}
\label{sec:intro}
%%% What is the human machine collaboration (what) and why it is important in the manufacturing system (need), what is the challenge in achieving the need? 

%There is an increasing need to provide higher autonomy of robots, vehicles and some other cyber-physical systems. However, operating in uncertain or even unknown environments while guaranteeing certain task-specification and optimality is in general fairly challenging.
\subsection{Background and Motivation}
Human-robot collaboration consists of intelligent robots and human operators collaborating with each other to accomplish complex tasks that are less efficient or cannot be completed by either human or robots alone. Human-robot collaboration has become an important and even necessary part to ensure safe and efficient operations for many safety-critical systems, to name a few, such as modern manufacturing processes \cite{sadrfaridpour2016modeling,chen2014optimal}, medical and health-care systems\cite{broadbent2009acceptance,okamura2010medical}, teleoperation \cite{lin2015experiments} and semi-autonomous vehicle systems \cite{seshia2015formal}. %, such as  modern manufacturing processes \cite{2017-HRI,sadrfaridpour2016modeling}, industrial processes \ref{} and semi-autonomous vehicle systems \cite{seshia2015formal}. 
Thus, a reliable and efficient optimal control framework is important, even necessary especially for systems with safety-critical applications \cite{zanchettin2016safety}. 

It is, however, fairly challenging to ensure desirable performance for human robot collaboration due to the complex interactions between human and robot. The complexity lies in (1) the inherent stochastic uncertainties caused by a variety of human characteristics and external environments (2) the difficulty of obtaining a priori knowledge about how human interacts with robot due to the unanticipated nature of human beings. To address these challenges, this paper proposes a learning based human-robot collaboration framework to ensure optimal system performance for a manufacturing task allocation system with linear temporal logic constraints (LTL) \cite{baier2008principles}, a rich and expressive specification language to specify desired properties that's close to human natural language.  %Assuring a reliable and efficient human robot collaboration, however,  due to the complexity brought by the human robot interaction. 
\subsection{Relevant Work}

Human-robot collaboration has been a rapidly growing research area that has found its roles in many applications\cite{bauer2008human}. It is thus, beyond the scope of this paper to conduct an exhaustive literature review on this popular topic. However, there are not many research activities that examine the optimal performance for human robot collaboration under LTL constraints. This section will focus on the discussions of the relationship and differences between the proposed work in this paper and the relevant research work.

Among many factors affecting human robot interaction, human factors, such as trust and fatigue, have been long considered as overarching concerns in human robot systems that significantly affects system performance, especially as it is related to safety-critical applications \cite{hancock2011meta}. Prior work in \cite{desai2012effects, robinette2015effect} showed that the evolution of both human trust and fatigue are dynamical processes. The human trust toward the robot is highly dependent on the automation performance. %Furthermore, the empirical data in \cite{desai2013impact,lee1992trust} also implied that the impact of machine faults on human trust is at least $10$ times bigger than the normal performance variations, and human trust level would drop abruptly due to machine malfunctions or faults. 
These observations motivate research efforts on the development of quantitative models for human trust
and fatigue to provide appropriate real time predictions \cite{lee1992trust,desai2013impact} and control of human trust levels \cite{sadrfaridpour2016modeling}. In this paper, we explicitly consider the impact of the human trust and fatigue on the human robot interaction. We assume the complete observability of the trust and fatigue level in real time, which can be met by using the techniques mentioned in \cite{wang2015mutual,Fu-RSS-14}.

To address the challenge of modeling the unknown uncertainties in human robot interaction, this paper proposes a composition framework of multiple Markov Decision Process (MDP) models with possibly unknown transition probabilities.  Our proposed framework can be viewed as a stochastic generalization of the conventional deterministic models proposed in \cite{lee1992trust,  sadrfaridpour2016modeling}. 
%Motivated by the inherent uncertainties in human robot interaction, we propose to use Markov Decision Process (MDP) to model human robot interaction. In particular, our proposed MDP model can be viewed as a stochastic abstraction of the conventional continuous models proposed in \cite{lee1992trust,  sadrfaridpour2016modeling}. MDP is a popular probabilistic model that is widely used in biology, robotics and economics \cite{ding2014optimal} and recently has also been used to predict the evolution of human cognitive and physiological states, see, e.g \cite{fujieSharedAutonomy,mcghan2015human,nikolaidis2015efficient}. 
Under the MDP models for human robot interactions,
the second challenge is to ensure the optimal performance
when conducting persistent
tasks, such as ``visiting regions A infinite often to pick up products" or ``repeating the steps of assembling the car engine infinite often", which are important and widely encountered situations in a variety of industrial systems. The desired performance metric for those
persistent tasks is to maximize the probability of successfully executing those persistent tasks while optimizing some cost metric.  Such performance specification is often
formulated as an optimal control problem of MDP with
temporal logic constraints. Wolff et. al \cite{wolff2012optimal} developed a method to automatically generate the optimal control policy subject to LTL specifications and optimizing the weighted average cost function on a non-probabilistic model. The controller synthesis framework to optimize average cost per operation cycle, a natural performance metric for persistent tasks with temporal logic constraints in MDPs was considered in \cite{ding2014optimal} and \cite{svorenova2013optimal}. All the above results assume that the model is precisely known. However, in many practical applications, such assumption may not be true and the model has to be learned through interaction with the environment.

Motivated by the need of considering unknown  models in human robot collaboration , this paper studies a learning-based optimal task allocation problem where the optimal policies can be synthesized on the learned model  to optimize the average cost per task cycle while maximizing the probability of satisfying the LTL constraints used to specify persistent tasks. The frameworks considered in \cite{sadigh2014learning,Fu-RSS-14} are the most relevant to the work considered in this paper. To address the optimal control problem under unknown MDPs, Sadigh et. al \cite{sadigh2014learning} developed a reinforcement learning based approach to synthesize control policy for MDP with unknown transition probabilities to satisfy LTL specification with optimal probability. Fu et. al \cite{Fu-RSS-14} extended the probably approximate correct (PAC)-MDP learning framework that synthesize controllers of unknown MDPs with temporal logic constraints. Both work focus on developing learning methods to maximize the probability of satisfying the LTL specification without considering the cost per task cycle metric. It is, however, unclear whether the methods in \cite{sadigh2014learning,Fu-RSS-14} can be used to achieve optimal average cost per task cycle. To our best knowledge, this paper presents the first results that address the optimization problem that achieves optimal average costs per task cycle with LTL constraints under the unknown MDP models. 

\subsection{Contributions}

The objective of this paper is to (1) develop a human robot collaboration framework that is able to model the unknown stochastic uncertainties on human robot interactions, (2) as well as a learning based algorithm to ensure optimal performance for human robot collaboration with LTL constraints under unknown human models. The main contributions of this paper are summarized as below,
\begin{itemize}
	\item This paper proposes a human robot collaboration framework with composition of multiple MDPs with unknown transition probabilities. Compared to the existing frameworks \cite{lee1992trust,desai2013impact,ji2006probabilistic,fujieSharedAutonomy}, such framework enables the characterization of stochastic uncertainties as well as their unknown features for human robot interactions as well as the use of formal design methodologies for performance guarantee.  
	\item Based on the unknown MDP models, this paper further develops a learning-based algorithm inspired by probably approximately correct method to achieve the optimality for both LTL satisfaction and average cost per task cycles. To the best of our knowledge, this is the first set of results addressing both issues of LTL constraints and average cost per task cycles under unknown MDP models. 
	\item This paper also shows that the optimal policies achieved by the proposed learning method under the unknown MDP model, can asymptotically approach the ones with the true model. The optimality gap of the proposed learning algorithm can be analytically bounded and to achieve such bound, the complexity to learn the model is polynomial in the size of the MDP, the size of automaton from the LTL specification and other quantities that either are task specific or measure the confidence level and accuracy of the learned model.
\end{itemize}

This paper extends our preliminary results in \cite{ACC2017} by considering synthesizing optimal policy under the model with unknown transition probabilities and prove its approximate optimality. %The main contribution of this paper is twofold. One is the proposed trust based human robot interaction model. The second is the design of an optimal task allocation strategy for human and robot under which the system specification is assured even if the model is not known initially. 
The rest of this paper is structured as follows. Section \ref{sec:preliminaries} provides the necessary preliminaries; Section \ref{sec: problem-formulation} models and formulates the problem. The task assignment problem is solved in Section \ref{sec:main_results} and Section \ref{sec:approxi} for known and unknown model respectively. A simulation is presented in Section \ref{sec:simulation}. Section \ref{sec: conclusion} concludes the paper. 
%In human-robot systems, the human robot interaction is often examined under two architectures. The first architecture, called peer-to-peer, characterizes the scenarios that  human operators work closely with robots in a shared environment. Under the peer to peer architecture, the system~(human) safety mostly arises due to human's or machine's inappropriate interpretations toward its partner's intentions. From machine's standpoint, such mis-interpretations can be partially addressed by estimating humans' intentions based on real time observations of their body movements.

%In human robot systems, the human robot interaction arises in two aspects. The first aspect involves the physical interactions between human and robots during operation processes. Such physical interaction often exists in scenarios where the human collaborates with robots in shared workspaces. Thus, to assure system~(human) safety, the human robot collaboration under this physical interaction, must avoid any physical collisions. 
%
%
%The other aspect of interaction comes from situations that the human collaborates with robots through remotely controlled computers. 

%%% Prior work on human robot collaboration, safety and efficiency issues, what is the difference between prior work and the proposed work

%%% Motivate the Trust based co-design methodology, review the prior work both from psychological community and engineering community, shows what is missing in those work

%%% What is the proposed work and the contribution?
\section{Preliminaries}\label{sec:preliminaries}
In this section, we introduce the required background knowledge about transition system, Markov Decision Process and temporal logic. 
\subsection{Transition System}
Transition system is a popular model to describe the non-probabilistic behavior of the system. For example, it can be used to describe the flow of the assembly task driven by worker's actions.
\begin{definition}\cite{baier2008principles}
	A transition system is a tuple $\mathcal{TS}=(S,A,\rightarrow,I,AP,L)$ where
	\begin{itemize}
		\item $S=\{s_0,s_1,...\}$ is a finite set of states;
		\item $A$ is a finite set of actions;
		\item $\rightarrow\subseteq S\times A\times S$ is a transition relation.
		\item $I\subseteq S$ is a set of initial states;
		\item $AP$ is a set of atomic propositions;
		\item $L:S\rightarrow2^{AP}$ is a labeling function that maps each $s\in S$ to one or several elements of $AP$
	\end{itemize}
\end{definition}
For $s\in S$ and $a\in A$, we denote $Post(s,a)=\{s'\in S|(s,a,s')\in \rightarrow\}$. Specifically, $\mathcal{TS}=(S,A,\rightarrow,I,AP,L)$ is called an action-deterministic transition system if $|I|\leq 1$ and $|Post(s,a)|\leq 1$ for all $s$ and $a$. 
\subsection{Markov Decision Process}

\begin{definition}\label{def:MDP}\cite{baier2008principles}
	An MDP with cost is a tuple $\mathcal{M}=(S,\hat{s},A,P,L,C)$ where
	\begin{itemize}
		\item $S=\{s_0,s_1,...\}$ is a finite set of states;
		\item $\hat{s}\in S$ is the initial state;
		\item $A$ is a finite set of actions;
		\item $P(s,a,s'):=Pr(s(i+1)=s'|s(i)=s,a(i)=a),~\forall i\geq0$.
		\item $L:S\rightarrow2^{AP}$ is the labeling function that maps each $s\in S$ to one or several elements of a set $AP$ of atomic propositions.
		\item $c:S\times A\rightarrow\mathbb{R}^+$ is the cost function that maps each state and action pair to a non-negative cost. The maximum cost is bounded by $R_{max}$.
	\end{itemize}
\end{definition}

For each state $s\in S$, we denote $A(s)$ as the set of available actions. From the definition it is not hard to see that the action-deterministic transition system is a special case of MDP with $P(s,a,s')=1$ for any $s\in S$ and $a\in A$ that are defined. Vice versa, if we ignore the probabilities in each transition, the MDP will become a transition system which we denote as $TS_\mathcal{M}$. 

A path $\omega$ of an MDP is a non-empty sequence of the form $\omega=s_0\xrightarrow{a_0}s_1\xrightarrow{a_1}s_2...s_i\xrightarrow{a_i}s_{i+1}...$, where each transition is enabled by an action $a_i$ such that $P(s_i,a_i,s_{i+1})>0$. We denote $Path^{fin}_s$ as the collection of finite length paths that start in a state $s$. The nondeterminism of an MDP is resolved with the help of a scheduler.
\begin{definition}
	A scheduler $\mu:Path_{\hat{s}}^{fin}\rightarrow A$ (also known as adversary or policy) of an MDP $\mathcal{M}$ is a function mapping every finite path $\omega_{fin}\in Path_{\hat{s}}^{fin}$ onto an action $a\in A(last(\omega_{fin}))$ where $last(\omega_{fin})$ denotes the last state of $\omega_{fin}$.
\end{definition}
\indent Intuitively, the scheduler specifies the next action to take for each finite path. The behavior of an MDP $\mathcal{M}$ under a given scheduler $\mu$ is purely probabilistic and thus reduces to a discrete time Markov chain (DTMC) with a set of recurrent classes. A recurrent class $S_{rc}$ refers to a set of recurrent states that the probability to return to any state $s\in S_{rc}$ after leaving it is $1$ and any pair $s,s'\in S_{rc}$ is communicating, that is,  $s'$ is reachable from $s$ (there exists a state sequence $s_0s_1s_2...s_n$ with nonzero probability such that $s_0=s,s_n=s'$) and $s$ is reachable from $s'$.  %We denote $Sch_{\mathcal{M}}$ as the set of all possible schedulers for $\mathcal{M}$. 
A policy $\mu$ is called memoryless if $\mu(\omega_{fin})=\mu(last(\omega_{fin})$, that is, the action to select only depends on the current state that the MDP is in. A policy is said to have memory otherwise. An MDP is said to be \emph{communicating} \cite{puterman2014markov} if there exits a memoryless and deterministic (stationary) policy such that the induced DTMC is communicating, that is , all the states pair $s,s'$ in the DTMC are communicating.   
\subsection{Linear Temporal Logic}
Under the MDP model, LTL  can be used to describe a wide range of properties of sequences of states such as safety (bad things never happen), liveness (good things eventually happen), persistence (good things happen infinitely often), response (if A then B) and so on. 

An LTL formula is built up from a set of atomic propositions $(AP)$, $true,false$, the Boolean operators $\neg$ (negation), $\vee$ (disjunction), $\wedge$ (conjunction) and temporal operators  $\square$ (always), $X$ (next), $\cup$ (until), $\diamondsuit$ (eventually). An LTL formula $\phi$ can always be represented by a deterministic Rabin automaton (DRA) \cite{baier2008principles} $\mathcal{R}_{\phi}=(Q,q_0, 2^{AP},\delta, ACC)$ where $Q$ is a set of finite states, $q_0\in Q$ is the initial state,  $2^{AP}$ is the alphabet, $\delta:Q\times 2^{AP}\rightarrow Q$ is the transition function, and $ACC=\{(L(1),K(1)),...,(L(M),K(M))\}$, with $M$ being a positive integer, is a set of tuples where $L(i),K(i)\subseteq Q$ for all $i=1,...,M$. Given a infinite word $l=\sigma_0\sigma_1...$ where $\sigma_i\in 2^{AP}$ for all $i$, a unique path $\omega=q_0q_1...$ will be induced where $q_{i+1}=\delta(q_i,\sigma_i)$. A run $\omega$ is \emph{{accepted}} in $\mathcal{R}_\phi$ if there exists a pair $(L,K)\in ACC$ such that 1) there exists $n\geq 0$, such that for all $m\geq n$, we have $q_m\notin L$, and 2) there exist infinitely many indices $k$ where $q_k\in K$. Note that $L$ could be empty but $K$ may not. Intuitively, for a pair $(L,K)$, the acceptance condition means that an accepted run should visit states in $L$ only finite times and states in $K$ infinite times. Given an LTL formula $\phi$, it is always possible to construct a DRA that accepts exactly the words that satisfy $\phi$.

\section{Modeling and problem formulation}
\label{sec: problem-formulation}
\subsection{Modeling}
\begin{figure}[h]
	\centering
	\includegraphics[scale=0.35]{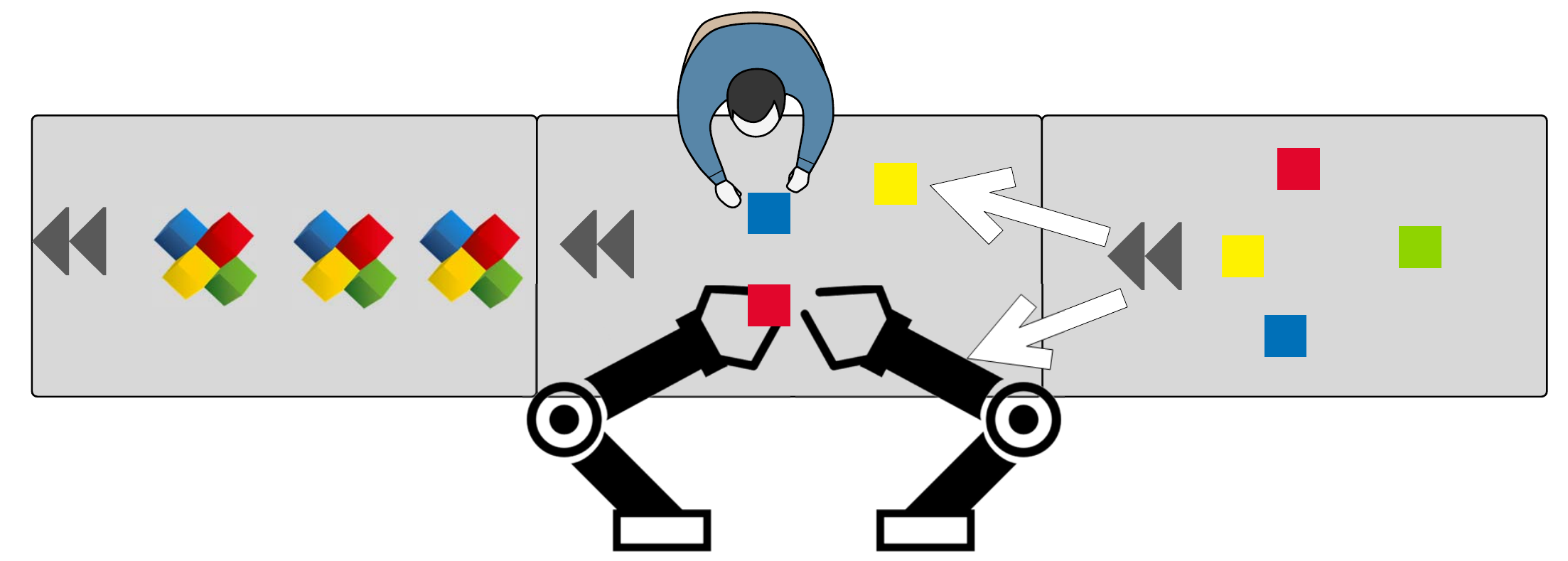}
	\caption{Human robot collaborative assembly}
	\label{fig:scenario}	
\end{figure}
%% Motivate and emphasize the importance of human machine interaction in terms of human trust, machine performance.
\noindent{\bf Task model}
Consider an industrial assembly process involving both human and robot as shown in Figure \ref{fig:scenario}. The assembly task can be represented by an action-deterministic transition system $\mathcal{M}^w$ as shown in Figure \ref{fig:spec}. There are $N$ parts that need to be assembled by actions $A^r=\{a^r_1,..,a^r_{N_r}\}$ denoting the robot actions, and $A^h=\{a^h_1,...,a^h_{N_h}\}$ denoting the human actions. In Figure \ref{fig:spec}, $a_i\in\{a_1,a_2\}$ can represent  either human or robot's corresponding action. % It specifies the order of the parts to be put together and potentially there are parts can be assembled in either order. For example in Figure \ref{fig:spec},  the assembly order of parts represented by $a_1$ and $a_2$ is interchangeable.

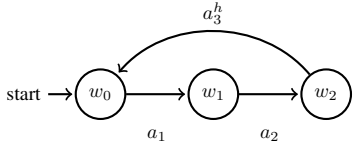
\begin{figure}[h]
	\centering	
	\begin{tikzpicture}[shorten >=1pt,node distance=2cm,on grid,auto, bend angle=50, thick,scale=.75, every node/.style={transform shape}]
	\node[state,initial] (s_0)   {$w_0$};
	\node[state] (s_1) [right=of s_0] {$w_1$};
	
	\node[state] (s_2) [right=of s_1] {$w_2$};

	%\node[state] (s_4) [below=of s_3] {};
	%	\node[state] (s_5) [left=of s_4] {};
	%	\node[state] (s_6) [left=of s_5] {};
	\path[->]
	(s_0) edge node [pos=0.5, sloped, below=0.5]{$a_1$} (s_1)
	
	(s_1) edge node [pos=0.5, sloped, below=0.5]{$a_2$} (s_2)
	
	(s_2) edge  [pos=0.5, bend right, above] node {$a_{3}^h$} (s_0)
	;
	%(s_3) edge node [pos=0.5, sloped, above]{$D_1 open$} (s_4)	
	%	(s_4) edge node [pos=0.5, sloped, above]{$R_2 to 1$} (s_5)
	%	(s_5) edge node [pos=0.5, sloped, above]{$R_2 in 1$} (s_6)
	%	(s_3) edge node [pos=0.5, sloped, above]{$r_1$} (s_0);	
	\end{tikzpicture}
	\caption{Assembly plan $\mathcal{M}^w$}
	\label{fig:spec}
	
\end{figure}	

During the task execution all necessary parts are fed by suitable mechanisms such as the conveyor belts. The two arrows in Figure \ref{fig:scenario} denote the distribution of the parts. As seen in Figure \ref{fig:spec}, these assembly tasks often consist of a sequence of stages where each stage represents one sub-task that needs human or machine's actions to complete.  In practice, the difficulty levels of the sub-tasks in the assembly process often varies from stages to stages, which may result in stochastic variations on machine's performance.

\noindent{\bf Robot Model} This paper introduces an MDP to model the stochastic dynamics of the machine performances. For each robot action $a_i^r$, robot may finish it with some cost representing the energy consumption and the time, additionally, $a_i^r$ may lead to a faulty state with certain probability $p_i$ to characterize the possible failure of the robot.  Once the robot is in the faulty state, human repair is required.

Formally the robot model is an MDP $\mathcal{M}_{r}=\{S^{r}, s_{0}^{r}, A^{r}\cup\{repair\}, {P}^{r}, L^r, c^r\}$ where $S^{r}$ is a finite set of states representing the machine state, $s_{0}^{r} \in S^{r}$ is the initial state, and ${T}^{r}$ is the transition matrix. Figure \ref{fig:robot} shows a simple example of two performance levels in a robotic system. In particular, the state labeled $Normal$ represents a normal performance level under which the actions taken by the robot can accomplish the tasks as expected, while the state labeled $Faulty$ denotes an abnormal status under which the robot will fail to performing any assigned tasks.
\begin{figure}[h]
	\centering	
	\begin{tikzpicture}[shorten >=1pt,node distance=4cm,on grid,auto, bend angle=30, thick,scale=.75, every node/.style={transform shape}]
	\node[state,initial] (s_0)   {$r_0$};
	\node[state] (s_1) [right=of s_0] {$r_1$};

	%\node[state] (s_4) [below=of s_3] {};
	%	\node[state] (s_5) [left=of s_4] {};
	%	\node[state] (s_6) [left=of s_5] {};
	\path[->]
	(s_0) edge [pos=0.5, bend left, above=0.5] node {$a^{r}_i,p_i$} (s_1)
	(s_0) edge [pos=0.5, loop, above=0.1] node {$a^{r}_i, 1-p_i$} (s_0)
	(s_1) edge [pos=0.5, bend left, below=0.5] node {$repair,1$} (s_0)
	;
	\node[text width=3cm] at (0.8,-1) 
	{\{normal\}};		
	\node[text width=3cm] at (4.8,-1) 
	{\{faulty\}};
	%(s_3) edge node [pos=0.5, sloped, above]{$D_1 open$} (s_4)	
	%	(s_4) edge node [pos=0.5, sloped, above]{$R_2 to 1$} (s_5)
	%	(s_5) edge node [pos=0.5, sloped, above]{$R_2 in 1$} (s_6)
	%	(s_3) edge node [pos=0.5, sloped, above]{$r_1$} (s_0);	
	\end{tikzpicture}
	\caption{Robot performance model}
	\label{fig:robot}
\end{figure}
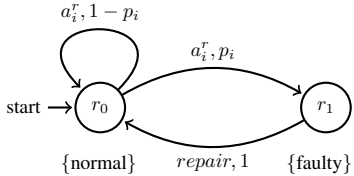	

%%% Discuss the model~(MDP) used to characterize the dynamics of human trust as a function of machine performance, as well as the evidence supporting the model 
\noindent{\bf Human trust and fatigue model} %To assure a reliable and efficient human robot collaboration, one must be able to characterize the dynamic interactions between human operator and robot.  Motivated by this need, we formally model the human robot interactions in the context of human trusts a MDP whose stochastic dynamics are governed by machine actions that implies its performance. 

%The robot model assesses how well the machine performs in each subtask. These assessments can be used later to infer the human trust toward the machine system, which has been well recognized as a key factor that profoundly affects human's performance in human machine collaboration. 
The interaction between human trust and robot performance is well known to exist in general human robot systems, such as process control systems \cite{muir1990operators} and human-robot collaboration system \cite{desai2013impact}. %Prior experimental results \cite{lee1992trust, desai2013impact} show that the human trust in automation system is a dynamic process that changes as a function of the robot performance. 
In particular, the work in \cite{lee1992trust, desai2013impact} demonstrated that the dynamics of the human trust can be adequately modeled as a first order linear system with the robot performance as its input. This finding motivates us to model the human trust as an MDP whose states represent different trust levels of the human operator. %  In particular, the robot action in which the robot performance is implied drives the changes of the human trust from one state to another. The MDP characterization of the human trust is defined as follows,

\begin{definition}[Human Trust Model]
	A human trust MDP model is a tuple $\mathcal{M}_{t}=\{S^{t}, s_0^t, A^t, {P}^{t}, c^{t}\}$ where
	\begin{itemize}
		\item $S^{t}$ is a finite set of human trust levels.
		\item $s^{t}_{0} \in S^{t}$ is the initial human trust level.
		\item $A^t=A^r\cup\{repair\}$ is the action set
		\item ${P}^{t}(s^t(i),a,s^{t}(i+1)=Pr(s^{t}(i+1)|s^t(i),a)$ for $a\in A^r\cup\{repair\}$.
		\item $c^{t}: S^{t} \times A^t \rightarrow \mathbb{R}^{+}$ is a positive cost function. 
	\end{itemize}
\end{definition} 
Similarly, it is also well known that the human performance is affected by the fatigue level \cite{ji2006probabilistic}. Thus, we model the human fatigue process as an MDP $\mathcal{M}_{f}=\{S^{f}, s_0^f, A^f, {P}^{f}, c^{f}\}$ where $A^f=A^h\cup A^r\cup \{repair\}$.

\noindent{\bf Human robot interaction model}. Under the definitions of the four models, namely task model $\mathcal{M}^w$, robot model $\mathcal{M}^r$, human trust model $\mathcal{M}^t$ and fatigue model $\mathcal{M}^f$, the human robot interaction considered in this paper is characterized by the fact that the human trust and fatigue can be appropriately regulated by controlling robot performance through strategic task assignment. A desired level of human trust and fatigue will certainly improve the human performance in the assembly process, thereby leading to an effective human-robot collaboration. %This observation motivates this paper to use such human robot interaction to design optimal joint-policies for human and robot systems under which a well-defined performance specification for a manufacturing system can be optimized. 
In the presence of human robot interaction, the manufacturing system is modeled as an MDP $\mathcal{M}$ from parallel composition of the four models. 
$$
\mathcal{M}=\mathcal{M}^w||\mathcal{M}^r||\mathcal{M}^t||\mathcal{M}^f
$$
Where $||$ denotes the parallel composition defined as follows.
\begin{definition}[parallel composition]\label{def:parallel}
	Given two Markov decision processes $\mathcal{M}_1=(S_1,s^1_0,A_1,P_1,L_1, c_1)$ and $ \mathcal{M}_2=(S_2,s^2_0,A_2,P_2,L_2, c_2)$, the parallel composition of $\mathcal{M}_1$ and $\mathcal{M}_2$ is the MDP $\mathcal{M}=\mathcal{M}_1||\mathcal{M}_2=(S_1\times S_2,s^1_0\times s^2_0,A_1\cup A_2,P,L,c)$ where $L(s_1,s_2)=L_1(s_1)\cup L_2(s_2)$, and  
	\begin{itemize}
		\item	$P((s_1,s_2),a,(s_1',s_2'))=P_1(s_1,a,s_1')P_2(s_2,a,s_2')$, $c((s_1,s_2),a)=c_1(s_1,a)+c_2(s_2,a)$, if $a\in A_1\cap A_2$ and both $P_1(s_1,a,s_1')$ and $P_2(s_2,a,s_2')$ are defined, or
		\item $P((s_1,s_2),a,(s_1',s_2))=P_1(s_1,a,s_1')$, $c((s_1,s_2),a)=c_1(s_1,a)$ if $a\in A_1\backslash A_2$ and $P_1(s_1,a,s_1')$ is defined, or
		\item $P((s_1,s_2),a,(s_1,s_2'))=P_2(s_2,a,s_2')$, $c((s_1,s_2),a)=c_2(s_2,a)$ if $a\in A_2\backslash A_1$ and $P_2(s_2,a,s_2')$ is defined.
	\end{itemize}
\end{definition}
From the definition of the parallel composition, it can be seen that the state $s$ of $\mathcal{M}$ is actually a four tuple $s=(s^w,s^r,s^t,s^f)$. Also note that the transition probability in the manufacturing process is not only dependent on the actions taken by the robot or human operator, but also dependent on the robot performance or the human trust level when those actions are taken. This dependency explicitly models the impact of the human robot interaction on the manufacturing process, and addresses the practical concern that the human or robot's real time performance does affect how well the manufacturing tasks being accomplished. Furthermore, we would like to point out that the quantification of the cost related to human models is application specific. Such metric may need to consider the time-to-completion, workload assessment,  comfort level and cognitive workload and so on \cite{fujieSharedAutonomy}. Detailed task analysis with subjective rating from human worker are needed for cost function construction. Inverse reinforcement learning  \cite{abbeel2004apprenticeship} may also be applied to learn a cost function based on human demonstrations of the human robot collaborative tasks.  
\subsection{Problem formulation} 
Once we get the system model $\mathcal{M}$, we are interested in how to optimally assign actions to human and robot dynamically based on their current states. % with the objective to optimize the cost. 

%In many cases, people are interested in getting the policy $\mu$ to minimize the following discounted accumulated rewards \cite{bertsekas1995dynamic}.
%$$
%J=\sum_{i}^{\infty}\gamma^i c(s_i,\mu)
%$$
%where $\gamma$ is a discount factor to make the accumulative cost finite and also implies that the algorithm favors the more immediate costs comparing to more distant costs. 
%
%However, in our scenario where assembly is a periodic procedure, discounting is inappropriate since we care about each round of assembly. 

Consider the composed MDP $\mathcal{M}$, suppose we are using the task model as shown in Figure \ref{fig:spec}, we label all the states of the form $S_{\pi}=\{w_0,*,*,*\}$ as $\pi$ where $*$ means arbitrary state in each submodule as appropriate. That is, $S_{\pi}$ denotes all the states when one round of assembly is finished. The assembly process should keep on going and the finished state should be visited infinite times.

We say that each visit to the set $S_\pi$ completes a cycle. That is, from the initial state, the first time to reach $s\in S_\pi$ is the first cycle, after that the path revisiting $s'\in S_\pi$ is the second cycle and so on. Given a path $\omega^s=s_0s_1...s_ns_{n+1}...$, then we denote the number of cycles completed up to stage $n$ as $C(\omega^s,n)$ which starts with $1$ at the very beginning. As we are dealing with an assembly task, the average cost per cycle (ACPC) is a more reasonable cost measure \cite{ding2014optimal} to optimize. Then we are ready to formally define our problem.

\begin{problem}\label{prob:main_problem}
	Given the composed MDP $\mathcal{M}$ and $S_\pi$, find a task assignment policy $\mu$ which optimizes the probability of finishing the assembly process infinitely often while minimizing the ACPC defined as
	\begin{equation}\label{eqn:optimal}
		\begin{split}
			J(s_0)=&\limsup_{N\rightarrow\infty}E\{\frac{\sum_{k=0}^{N}c(s_k,\mu(\omega_\mu^{s_k}))}{C(\omega_\mu^s,N)}|S_\pi\text{ is visited}\\ &\text{infinitely often and $L(w^s_\mu)$ obeys other constraints}\}
		\end{split}
	\end{equation} 
	where $\omega_\mu^s$ denotes the state path under policy $\mu$ and $\omega_\mu^{s_k}$ denotes the state path up to stage $k$. $L(\omega_\mu^s)$ is the observed label sequence along the $\omega_\mu^s$. Other constraints may impose additional requirements on $\mu$ as well, such as the robot and human assembly action should be alternatively assigned.
\end{problem}

\section{Dynamic task assignment with known model}\label{sec:main_results}
The requirement that the assembly process should be finished infinite times with additional constraints can be written into the following LTL formula
\begin{equation}\label{eqn:LTL}
	\phi=\square\diamondsuit\pi\wedge\varphi
\end{equation}
where $\varphi$ denotes other constraints. Therefore, we can convert the Problem \ref{prob:main_problem} into the following.
\begin{problem}\label{prob:LTL_problem}
	Given the composed MDP $\mathcal{M}$, LTL formula $\phi$ in (\ref{eqn:LTL}) and $S_\pi$, find a policy $\mu$ which optimizes the probability of satisfying $\phi$ while minimizing
	\begin{equation}\label{eqn:optimalltl}
		J(s_0)=\limsup_{N\rightarrow\infty}E\{\frac{\sum_{k=0}^{N}c(s_k,\mu(\omega_\mu^{s_k}))}{C(\omega_\mu^s,N)}|L(\omega_\mu^s)\models\phi\}
	\end{equation} 
\end{problem}

Given the task model, robot model, human trust model and human fatigue model as MDPs, we compose them using parallel composition introduced in Definition \ref{def:parallel} into the system model $\mathcal{M}$. The specification is given in terms of an LTL formula $\phi$ and is converted into a DRA $\mathcal{R}_\phi$. An example of the DRA corresponds to the LTL formula $\phi=\square\diamondsuit\pi$ is shown in Figure \ref{fig:dra} where $K=\{q_1\}$ and is marked in double circles. So the accepted run in $R_\phi$ should visit $q_1$ infinitely often.

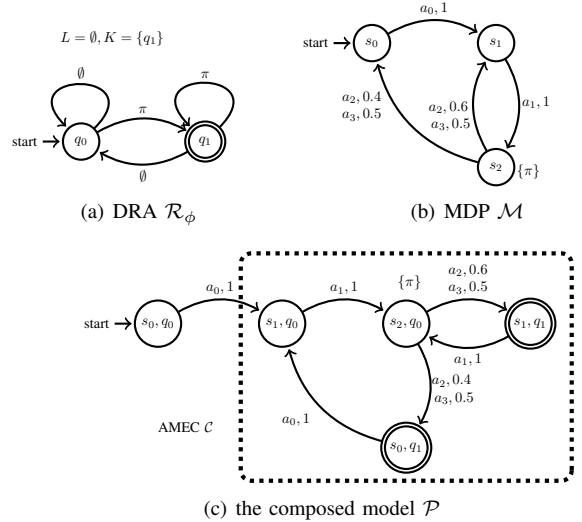
\begin{figure}[h]
	\centering
	
	\subfigure[DRA $\mathcal{R}_\phi$]{\label{fig:dra} 
		\begin{tikzpicture}[shorten >=1pt,node distance=3cm,on grid,auto, bend angle=30, thick,scale=0.55, every node/.style={transform shape}]
		\node[state,initial] (s_0)   {$q_0$};
		\node[state,accepting] (s_1) [right=of s_0] {$q_1$};

		%\node[state] (s_4) [below=of s_3] {};
		%	\node[state] (s_5) [left=of s_4] {};
		%	\node[state] (s_6) [left=of s_5] {};
		\path[->]
		(s_0) edge [pos=0.5, bend left, above=0.5] node {$\pi$} (s_1)
		(s_0) edge [pos=0.5, loop, above=0.1] node {$\emptyset$} (s_0)
		(s_1) edge [pos=0.5, bend left, below=0.5] node {$\emptyset$} (s_0)
		(s_1) edge [pos=0.5, loop, above=0.1] node {$\pi$} (s_1)
		;
		
		%(s_3) edge node [pos=0.5, sloped, above]{$D_1 open$} (s_4)	
		%	(s_4) edge node [pos=0.5, sloped, above]{$R_2 to 1$} (s_5)
		%	(s_5) edge node [pos=0.5, sloped, above]{$R_2 in 1$} (s_6)
		%	(s_3) edge node [pos=0.5, sloped, above]{$r_1$} (s_0);	
		\node[text width=3cm] at (1,2.5) {$L=\emptyset,K=\{q_1\}$};	
		\end{tikzpicture}}
	\subfigure[MDP $\mathcal{M}$]{\label{fig:mdp} 
		\begin{tikzpicture}[shorten >=1pt,node distance=3cm,on grid,auto, bend angle=30, thick,scale=0.55, every node/.style={transform shape}]
		\node[state,initial] (s_0)   {$s_0$};
		\node[state] (s_1) [right=of s_0] {$s_1$};
		\node[state] (s_2) [below=of s_1] {$s_2$};		
		
		%\node[state] (s_4) [below=of s_3] {};
		%	\node[state] (s_5) [left=of s_4] {};
		%	\node[state] (s_6) [left=of s_5] {};
		\path[->]
		(s_0) edge [pos=0.5, bend left, above=0.5] node {$a_0,1$} (s_1)
		(s_1) edge [pos=0.6, bend left, above=0.5] node {$~~~~~~~a_1,1$} (s_2)
		(s_2) edge [pos=0.5, bend left, above=0.5] node [align=center] {$a_2,0.4$~~~~~~~~~~~~~~~~~~~~~\\$a_3,0.5$~~~~~~~~~~~~~~~~~~~~~} (s_0)
		(s_2) edge [pos=0.6, bend left, below=0.5] node [align=center] {$a_2,0.6$~~~~~~~~~~~\\$a_3,0.5$~~~~~~~~~~} (s_1)
		;
		
		%(s_3) edge node [pos=0.5, sloped, above]{$D_1 open$} (s_4)	
		%	(s_4) edge node [pos=0.5, sloped, above]{$R_2 to 1$} (s_5)
		%	(s_5) edge node [pos=0.5, sloped, above]{$R_2 in 1$} (s_6)
		%	(s_3) edge node [pos=0.5, sloped, above]{$r_1$} (s_0);	
		\node[text width=3cm] at (5,-3.1) 
		{\{$\pi$\}};
		\end{tikzpicture}}
	\subfigure[the composed model $\mathcal{P}$]{\label{fig:mdp+dra} 
		\begin{tikzpicture}[shorten >=1pt,node distance=3cm,on grid,auto, bend angle=30, thick,scale=0.55, every node/.style={transform shape}]
		\node[state,initial] (s_0)   {$s_0,q_0$};
		\node[state] (s_1) [right=of s_0] {$s_1,q_0$};
		\node[state] (s_2) [right=of s_1] {$s_2,q_0$};
		\node[state,accepting] (s_3) [right=of s_2] {$s_1,q_1$};
		\node[state,accepting] (s_4) [below=of s_2] {$s_0,q_1$};		
		
		%\node[state] (s_4) [below=of s_3] {};
		%	\node[state] (s_5) [left=of s_4] {};
		%	\node[state] (s_6) [left=of s_5] {};
		\path[->]
		(s_0) edge [pos=0.5, bend left, above=0.5] node {$a_0,1$} (s_1)
		(s_1) edge [pos=0.5, bend left, above=0.1] node {$a_1,1$} (s_2)
		(s_2) edge [pos=0.5, bend left, above=0.1] node [align=center] {$a_2,0.6$\\$a_3,0.5$} (s_3)
		(s_3) edge [pos=0.5, bend left, below=0.5] node {$a_1,1$} (s_2)
		(s_2) edge [pos=0.8, bend left, above=0.5] node [align=center] {~~~~~~~~~~$a_2,0.4$\\~~~~~~~~~~~$a_3,0.5$} (s_4)
		(s_4) edge [pos=0.5, bend left, below=0.5] node {$a_0,1~~~~~~~~~~$} (s_1)
		;
		\draw[black,ultra thick,rounded corners,dotted] (10,-3.8) rectangle (2,1.7);
		
		%(s_3) edge node [pos=0.5, sloped, above]{$D_1 open$} (s_4)	
		%	(s_4) edge node [pos=0.5, sloped, above]{$R_2 to 1$} (s_5)
		%	(s_5) edge node [pos=0.5, sloped, above]{$R_2 in 1$} (s_6)
		%	(s_3) edge node [pos=0.5, sloped, above]{$r_1$} (s_0);	
		\node[text width=3cm] at (1.5,-2.5) 
		{AMEC $\mathcal{C}$};
		\node[text width=3cm] at (7.3,1) 
		{\{$\pi$\}};
		\end{tikzpicture}}
	
	\caption{An example of DRA $\mathcal{R}_\phi$ for $\phi=\square\diamondsuit\pi$, the MDP $\mathcal{M}$, and the construction of the composed model $\mathcal{P}$}
	\label{fig:example_DRA}
\end{figure}

Then the product $\mathcal{P}$ between the MDP $\mathcal{M}$ and Rabin automaton $\mathcal{R}_\phi$  is needed to capture all the paths of $\mathcal{M}$ to satisfy $\phi$.

%\begin{figure}
%	\centering
%	\includegraphics[scale=1]{../fig/framework}
%	\caption{The overall framework}
%	\label{fig:framework}	
%\end{figure}

\begin{definition}\label{def:DRA_comp} \cite{ding2014optimal}
	Given an MDP $\mathcal{M}=\{S , s_{0} , A , P, \mathcal{AP}, \mathcal{L}, c\}$ and a DRA $\mathcal{R}_{\phi}=(Q,q_0, 2^{AP},\delta, ACC)$, the product is an MDP $\mathcal{P}=(S_{\mathcal{P}},s_0\times q_0, A,P_{\mathcal{P}}, ACC_{\mathcal{P}} , S_{\mathcal{P}\pi}, c_\mathcal{P})$ where
	
	\begin{itemize}
		\item $S_{\mathcal{P}}=S \times Q$
		\item$ P_{\mathcal{P}}((s,q),a,(s',q')=Pr(s'|s,a)$ if $q'=\delta(q,\mathcal{L}(s))$, 0 otherwise;
		\item $ACC_{\mathcal{P}}=\{(L_\mathcal{P}(1),K_\mathcal{P}(1)),...,(L_\mathcal{P}(M),K_\mathcal{P}(M))\}$ where $L_\mathcal{P}(i)=S\times L(i)$, $K_\mathcal{P}(i)=S\times K(i)$, for $i=1,...,M$;
		\item	$S_{\mathcal{P}\pi}=S_{\pi}\times Q$;
		\item $c_\mathcal{P}((s,q),a)=c(s,a)$.
	\end{itemize}
\end{definition}

A simple MDP $\mathcal{M}$ is shown in Figure \ref{fig:mdp} and its product MDP $\mathcal{P}$ is illustrated in Figure \ref{fig:mdp+dra} where $L_\mathcal{P}=\emptyset,K_\mathcal{P}=\{(s_1,q_1),(s_0,q_1)\}$ which are marked in double circles. Note that there is a one-to-one correspondence between a path $s_0s_1,...$ on $\mathcal{M}$ and a path $(s_0,q_0)(s_1,q_1)...$ on $\mathcal{P}$ with the same cost according to the Definition \ref{def:DRA_comp} which enables us to only focus on the behavior of the composed model $\mathcal{P}$. If there is a memoryless policy on $\mathcal{P}$, it is always possible to map it back to $\mathcal{M}$ and possibly become a policy with memory which is from the underlying DRA $\mathcal{R}_\phi$.

Then the problem of maximizing the probability of satisfying the LTL formula $\phi$ for $\mathcal{M}$ can be transformed into a problem of maximizing the probability of reaching a set $\mathcal{C}(\mathcal{P})$ of Accepting Maximal End Components (AMEC) in the product MDP $\mathcal{P}$ \cite{baier2008principles}.  

\begin{definition}\cite{ding2014optimal}
	Given a pair $(L_\mathcal{P},K_{\mathcal{P}})$, an end component $\mathcal{C}$ is a communicating MDP $(S_\mathcal{C},A_\mathcal{C},P_{\mathcal{C}},K_\mathcal{C},S_{\mathcal{C}\pi},c_{\mathcal{C}})$ where $S_\mathcal{C}\subseteq S_\mathcal{P}$,$A_\mathcal{C}\subseteq A$, $A_\mathcal{C}(s)\subseteq A_\mathcal{P}(s)$ for all $s\in S_\mathcal{C}$, $K_\mathcal{C}=S_\mathcal{C}\cap K_\mathcal{P}$, $S_{\mathcal{C}\pi} = S_\mathcal{C}\cap S_{\mathcal{P}\pi}$, and $c_{\mathcal{C}}(s,a)= c_{\mathcal{P}}(s,a)$ for $s\in S_\mathcal{C},a\in A_\mathcal{C}(s)$. For any $P(s,a,s')>0$, $s\in S_\mathcal{C},a\in A_\mathcal{C}(s)$, $s'\in S_\mathcal{C}$ and $P_\mathcal{C}(s,a,s')=P_\mathcal{P}(s,a,s')$. An Accepting Maximal End Components (AMEC) is the largest such en component with  nonempty $K_\mathcal{C}$ and $S_\mathcal{C}\cap L_\mathcal{P}=\emptyset$.
\end{definition}

AMECs with respect to the same acceptance pair are pairwise disjoint. Therefore the total number of AMECs is bounded by $|S||Q||ACC_\mathcal{P}|$. There exists at least one such AMEC and the procedure of finding the AMEC set $\mathcal{C}(\mathcal{P})$ can be found in \cite{baier2008principles}. As shown in  Figure \ref{fig:mdp+dra}, there is only one AMEC $\mathcal{C}$ in $\mathcal{C}(\mathcal{P})$ which is inside the dotted box. Once $\mathcal{C}$ is reached, it is always possible to construct a policy such that some $K\in K_{\mathcal{P}}$ is visited infinitely often and thus the LTL constraint is satisfied. Therefore the maximum probability of satisfying $\phi$ is equal to the maximum probability of reaching $\mathcal{C}(\mathcal{P})$. We can find a stationary policy $f_{\rightarrow\mathcal{C}}$ via value iteration or linear program such that the probability of reaching a state in one of the AMEC $\mathcal{C}\in\mathcal{C}(\mathcal{P})$  equals $P_{max}(\phi)$. Note that $f_{\rightarrow\mathcal{C}}$ is defined only outside of $\mathcal{C}$.
%
%\begin{definition} \cite{baier2008principles}
%Accepting Maximal End Components
%	
%\end{definition}

Now we focus on minimizing the average cost per cycle (ACPC). From (\ref{eqn:optimalltl}) we know that as $N$ goes to infinity, the cost contributed by the finite path from the initial state to a state in an AMEC would be zero and $J(s_0)$ is actually solely decided by the policy $f_{\circlearrowright\mathcal{C}}$ which is defined inside of the AMEC. Once inside $\mathcal{C}$, the task is then to construct a stationary policy $f_{\circlearrowright\mathcal{C}}$ such that a state in $K_{\mathcal{P}}$ is visited infinitely often and minimize (\ref{eqn:optimalltl}). We can utilize Algorithm 1 in \cite{ding2014optimal}, a policy iteration algorithm to get the optimal (or suboptimal) $f_{\circlearrowright\mathcal{C}}$ for each $\mathcal{C}$. Then we find the optimal policy $f_\mathcal{C}$ for all $\mathcal{C}\in\mathcal{C}(\mathcal{P})$ such that 
\[
f_\mathcal{C}= 
\begin{cases}
f_{\rightarrow\mathcal{C}}(i),& \text{if } i\notin S_\mathcal{C}\\
f_{\circlearrowright\mathcal{C}}(i)             & \text{otherwise}
\end{cases}
\]
The Problem \ref{prob:LTL_problem} is then solved by finding optimal policy $f^*$ from $f_\mathcal{C}^*$ such that it incurs the optimal ACPC among all $f_\mathcal{C}$. Note that this stationary policy is defined on $\mathcal{P}$, whose state is the tuple $(s,q)$ where $q$ is the DRA's state. Therefore different actions could be assigned for the same $s$ on $\mathcal{M}$ since $q$ may be different and it actually induces a policy with finite memory for $\mathcal{M}$.

%Furthermore, we prove the optimality starting from the following lemma. 
%\begin{lemma}\label{lem:unichain}\cite{ding2014optimal}
%	For each AMEC $\mathcal{C}$ in the AMEC sets of $\mathcal{P}$, if $\mathcal{C}$ is unichain, then Algorithm 1 in \cite{ding2014optimal} is guaranteed to return the optimal policy $\mu^*_{\circlearrowright\mathcal{C}}$.
%\end{lemma}
%Then we are ready to give the following theorem.
%\begin{theorem}
%	Given the composed model $\mathcal{P}$ with each submodule MDP modeled as in Section \ref{sec: problem-formulation}, and the LTL formula $\phi$, Algorithm 1 in \cite{ding2014optimal} will always return the optimal policy.
%\end{theorem}
%\begin{proof}
%	{\color{red}{to be added}}
%\end{proof}

\section{Dynamic task assignment with unknown model}\label{sec:approxi}
For human robot interaction applications, in many cases the  model may not be known a prior and has to be learned. This section discusses finding the optimal control policy under an unknown MDP model.

With the unknown MDP model $\mathcal{M}$, LTL specification $\phi$, the composed MDP $\mathcal{P}$, We have the following assumptions. 
\begin{assumption}\label{assump:structure}
	The structure of $\mathcal{M}$, e.g. $TS_{\mathcal{M}}$ is known, but the transition probability is not known.
\end{assumption}
\begin{assumption}\label{assump:loop}
	Every loop in the form of $s_0s_1...s_ns_0$ induced by any policy in $\mathcal{C}\in \mathcal{C}(\mathcal{P})$ of $\mathcal{P}$ includes at least one state $(s,q)$ where $s\in S_\pi$ and the longest number of steps to finish a cycle starting at $(s,q)$ and ending at $(s',q')$ where $s,s'\in S_\pi$ is bounded by an finite positive integer $L$.
\end{assumption}
\begin{assumption}\label{assump:AMEC}
	Given the LTL formula $\phi$ and the MDP $\mathcal{M}$, $P_{max}(\phi)=1$ and for every AMEC $\mathcal{C}\in \mathcal{C}(\mathcal{P})$, there exists deterministic and memoryless policies such that $P(\diamondsuit \mathcal{C}) =1$ under this policy. Furthermore, the first state $(s,q)\in \mathcal{C}$ for all such polices visit is unique and $s\in S_\pi$. We can such a state the entrance of $\mathcal{C}$ and denoted by $ent(\mathcal{C})$. 
\end{assumption}
\begin{assumption}\label{assump:LTL}
	Every $\mathcal{C}\in \mathcal{C}(\mathcal{P})$ is a unichain MDP \cite{puterman2014markov}, that is, for every deterministic and memoryless policy in $\mathcal{C}$, it will incur single recurrent class.
\end{assumption}
\begin{remark}
	In many practical scenarios, the underlining structure of the MDP models can be known beforehand. For example, the assembly process model, the robot performance model and so on. Therefore, Assumption 1 is not overly restrictive. Assumption \ref{assump:loop} makes sure that no matter which action is chosen, the system always makes progress towards finishing a cycle. Therefore the number of steps to finish a cycle is bounded. With Assumption \ref{assump:AMEC}, we can examining the optimality in each $\mathcal{C}\in \mathcal{C}(\mathcal{P})$ and designs the policy to reach $\mathcal{C}$ with the optimal ACPC. Also, since the first state in $\mathcal{C}$ to be visited is fixed, it facilitates us to compare the optimality of bounded time polices in different $\mathcal{C}$'s. Assumption \ref{assump:LTL} is similar to the one stated \cite{ding2014optimal} in the sense that we guarantee that the optimal policy also satisfies the LTL specification.
\end{remark}

\begin{problem}\label{prob:unknown_problem}
	Given the MDP model $\mathcal{M}$ with unknown transition probability, and LTL specification $\phi$, two parameters $\delta,\epsilon$ with Assumptions \ref{assump:structure}, \ref{assump:loop}, \ref{assump:AMEC} and \ref{assump:LTL}, design an algorithm such that with probability at least $1-\delta$, find a policy $f_{\rightarrow\mathcal{C}}$ for every $\mathcal{C}\in\mathcal{C}(\mathcal{P})$ such that $\mathcal{C}$ can be reached with probability 1. Furthermore, find a policy $f_{\circlearrowleft\mathcal{C}} :S\times Q\rightarrow A$  defined in one of the AMEC $\mathcal{C}\in\mathcal{C}(\mathcal{P})$ such that once enters this $\mathcal{C}$, the $T$ cycle ACPC is $\epsilon-$close to the optimal ACPC in $\mathcal{M}$. And the sample complexity is polynomial in $|S|,|A|,T,\frac{1}{\epsilon},\frac{1}{\delta}$ and linear in the number of accepting pairs. 
\end{problem}

From the problem formulation, we are only interested in the ACPC after reaching the AMEC. To solve Problem \ref{prob:unknown_problem}, we first define $\alpha-$approximation in MDPs by combining the definitions in \cite{brafman2002r} and \cite{Fu-RSS-14} to consider both labels and rewards.
\begin{definition}\label{dfn:alphaMDP}
	Let $\mathcal{M}=(S,\hat{s},A,P,L,C)$ and $\bar{\mathcal{M}}=(\bar{S},\bar{\hat{s}},\bar{A},\bar{P},\bar{L},\bar{C})$ be two labeled MDPs as defined in Definition \ref{def:MDP}. We say that $\bar{\mathcal{M}}$ is an $\alpha-$approximation of $\mathcal{M}$ if
	\begin{itemize}
		\item $S=\bar{S}$, $\hat{s}=\bar{\hat{s}}$, $A=\bar{A}$, $P(s,a,s')>0$ iff $\bar{P}(s,a,s')>0$, $L=\bar{L}$ and $C=\bar{C}$. That is, they share the same state and action space, initial condition, structure, labeling and reward function.
		\item $|P(s,a,s')-\bar{P}(s,a,s')|\leq\alpha$ for any $s,s'$ and $a$.
	\end{itemize}
\end{definition}
By definition, it is not hard to see that if $\bar{\mathcal{M}}$ is an $\alpha-$approximation of $\mathcal{M}$, then $\bar{\mathcal{P}}=\bar{\mathcal{M}}\times\mathcal{R}_{\phi}$ is an $\alpha-$approximation of $\mathcal{P}=\mathcal{M}\times\mathcal{R}_{\phi}$. Since we assume that $P_{max}(\mathcal{M}\models\phi)=1$, we then have the following lemma. 
\begin{lemma}
	If $\bar{\mathcal{M}}$ is an $\alpha-$approximation of $\mathcal{M}$, $\phi$ is an LTL formula and  $P_{max}(\mathcal{M}\models\phi)=1$, then $P_{max}(\bar{\mathcal{M}}\models\phi)=1$.
\end{lemma}
\begin{proof}
	Since $\bar{\mathcal{M}}$ is an $\alpha-$approximation of $\mathcal{M}$, we know that $\bar{\mathcal{P}}=\bar{\mathcal{M}}\times\mathcal{R}_{\phi}$ is an $\alpha-$approximation of $\mathcal{P}=\mathcal{M}\times\mathcal{R}_{\phi}$ where $\mathcal{R}_{\phi}$ is then DRA from $\phi$. By definition, $P_{max}(\mathcal{M}\models\phi)=1$ implies that there exists a policy $\mu$ such that $P_\mu(\diamondsuit\mathcal{C}(\mathcal{P}))=1$ . Now we prove that $ P_{max}(\bar{\mathcal{M}}\models\phi)=1$ with the same policy $\mu$ by contradiction. 
	
	First, observe that $\mathcal{C}(\mathcal{P})$ is identical with $\mathcal{C}(\bar{\mathcal{P}})$ expect for the transition probabilities since $\mathcal{P}$ and $\bar{\mathcal{P}}$ share the same structure. Suppose under the same $\mu$ such that $P_\mu(\diamondsuit\mathcal{C}(\mathcal{P}))=1$, there exist paths with nonzero probability in $\bar{\mathcal{P}}$ that $\mathcal{C}(\bar{\mathcal{P}})$ is never reached, those paths must also exist with nonzero probability in $\mathcal{P}$ under $\mu$ because of the same structure, which is not possible since we know $ P_{max}(\mathcal{M}\models\phi)=1$.
\end{proof}

Once we have an estimated model that is close enough to the true MDP model, we have the following lemma to evaluate the performance of the synthesized optimal policy, which is similar to the simulation lemma in \cite{brafman2002r} but we extended it from T-step horizon to T-cycle horizon.
\begin{lemma}\label{lemma:main}
	If $\bar{\mathcal{M}}$ is a $\frac{\epsilon}{NTR_{max}D^2}-$approximation of $\mathcal{M}$, $N=|S|$, $T$ is a finite number of cycles and $0<\epsilon<1$, $D$ is the longest number of steps to finish a cycle starting from any finishing state under any proper policy that can complete $T$ cycle with probability $1$, then for any LTL specification $\phi$, for the T-cycle policy $f :S\times Q\rightarrow A
	$
	defined in $\mathcal{C}$, an AMEC, we have that $|J^{f,T}_{\mathcal{P}}(s,q)-J^{f,T}_{\bar{\mathcal{P}}}(s,q)|\leq\epsilon$, where $(s,q)=ent(\mathcal{C})$ and $J^{f,T}$ is the average cost in $T$ cycles with the policy $f$.  
\end{lemma} 
\begin{proof}
	Under the T-cycle policy $f$, we have that
	$$
	|J^{f,T}_{\mathcal{P}}(s,q)-J^{f,T}_{\bar{\mathcal{P}}}(s,q)| \leq \sum_{\omega}|P_{\mathcal{P}}(\omega) J_{\mathcal{P}}(\omega)-P_{\bar{\mathcal{P}}} (\omega) J_{\bar{\mathcal{P}}}(\omega)|
	$$
	where $P(\omega)$ and $J(\omega)$ are the probability of the path $\omega$ that runs for exactly $T$ cycles  and its corresponding average cost in $T$ cycles under policy $f$. Since $f$ is proper, we have that $\sum_{\omega}P_{\mathcal{P}} (\omega)= \sum_{\omega}P_{\bar{\mathcal{P}}}(\omega)=1$. Each such $\omega$ starts in a state $(s,q)$ and ends in a state $(s',q')$ where $s,s'\in S_\pi$ and runs for exactly $T$ cycles. Note that the paths may be of different lengths by going through $T$ cycles, but their length are bounded by $DT$. To facilitate our analysis, we (virtually) extend all the paths to $DT$. Once a path $\omega$ finishes $T$ cycles but has not reached $DT$ steps, it stays at the last state with probability $1$ and receives $0$ cost from then on until the path is $DT$ steps long. Furthermore, the cycle counts will also stay at $T$. In this way, neither the probability nor the average cost will be affected. 
	
	From Definition \ref{dfn:alphaMDP}, it can be seen that 
	$$
	J_{\mathcal{P}}(\omega)=J_{\bar{\mathcal{P}}}(\omega)=\frac{\sum_{k=1}^{DT}c_k}{T}\leq \frac{DTR_{max}}{T} = DR_{max}
	$$
	Therefore, we need to show that
	$$
	\sum_{\omega}|P_{\mathcal{P}}(\omega) -P_{\bar{\mathcal{P}}} (\omega) |\leq \frac{\epsilon}{DR_{max}}
	$$
	
	We define the following random processes $\mathcal{P}_i$ by following the policy $f$ with all paths extended to $DT$ steps, the first $i$ transitions are the same as in $\bar{\mathcal{P}}^f$ and the rest transitions are the same as in $\mathcal{P}^f$, where $\mathcal{P}^f$ and $\bar{\mathcal{P}}^f$ are DTMCs induced by the policy $f$ from $\mathcal{P}$ and $\bar{\mathcal{P}}$ Clearly, we have that $\bar{\mathcal{P}}^f=\mathcal{P}_{DT}$ and $\mathcal{P}^f=\mathcal{P}_0$. Then we have that
	\begin{equation} \label{eqn:sum}
		\begin{split}
			&\sum_{\omega}|P_{\mathcal{P}}(\omega)-P_{\bar{\mathcal{P}}} (\omega) |=\sum_{\omega}|P_{\mathcal{P}_0}(\omega) -P_{\mathcal{P}_{DT}} (\omega) |\\
			&=\sum_{\omega}|P_{\mathcal{P}_0}(\omega)-P_{\mathcal{P}_1}(\omega)+P_{\mathcal{P}_1}(\omega)-P_{\mathcal{P}_2}(\omega)+...\\
			&+P_{\mathcal{P}_{DT-1}} (\omega) -P_{\mathcal{P}_{DT}} (\omega) |\\
			&\leq \sum_{i=0}^{DT-1}\sum_{\omega}|P_{\mathcal{P}_{i}} (\omega) -P_{\mathcal{P}_{i+1}}(\omega)|
		\end{split}
	\end{equation}
	Then what's left to prove is that 
	\begin{equation}\label{eqn:desired}
		\sum_{\omega}|P_{{i}} (\omega) -P_{{i+1}}(\omega)|\leq\frac{\epsilon}{TR_{max}D^2}
	\end{equation}
	
	where for notational simplicity, we rewrite $P_{\mathcal{P}_{i}}$ as $P_i$. Inspired by \cite{brafman2002r},  if $s_i$ denotes the state that is reachable after $i$ transitions, $pre(s_i)$ and $suf(s_i)$ denote the $i$-step prefix reaching $s_i$ (not including $s_i$) and the suffix starting at $s_i$, respectively, for some $\omega = pre(s_i).s_i.suf(s_{i+1})$, we have 
	$$
	P_i(\omega) = P_i(pre(s_i))P_i(s_{i+1}|s_i)P_i(suf(s_{i+1}))
	$$. Then (\ref{eqn:sum}) can be rewritten as 
	\begin{equation} \label{eqn:rewritten}
		\begin{split}
			&\sum_{\omega}|P_{{i}} (\omega) -P_{{i+1}}(\omega)| = \sum_{s_i}\sum_{pre(s_i)}\sum_{s_{i+1}}\sum_{suf(s_{i+1})}\\
			& |P_i(pre(s_i))P_i(s_{i+1}|s_i)P_i(suf(s_{i+1}))-\\
			& P_{i+1}(pre(s_{i}))P_{i+1}(s_{i+1}|s_i)P_{i+1}(suf(s_{i+1}))|
		\end{split}
	\end{equation}
	Note that (\ref{eqn:rewritten}) holds even if we virtually extended all the $\omega$ to the same length. Also from the definition of $\mathcal{P}_i$, it is not hard to observe that $P_i(pre(s_i)) = P_{i+1}(pre(s_i))$ and $P_i(suf(s_{i+1}))=P_{i+1}(suf(s_{i+1}))$. Therefore (\ref{eqn:rewritten}) becomes
	\begin{equation} 
		\begin{split}
			& \sum_{s_i}\sum_{pre(s_i)}P_i(pre(s_i))\sum_{s_{i+1}}\sum_{suf(s_{i+1})} P_i(suf(s_{i+1})) |P_i(s_{i+1}|s_i)\\
			& -P_{i+1}(s_{i+1}|s_i)|\leq \sum_{s_i}\sum_{pre(s_i)}P_i(pre(s_i))\sum_{s_{i+1}}\sum_{suf(s_i)}\\ &P_i(suf(s_{i+1}))\frac{\epsilon}{NTR_{max}D^2}
		\end{split}
	\end{equation}
	Note that $\sum_{s_i}\sum_{pre(s_i)}P_i(pre(s_i))=1$ since it sums over all prefix of $i$ steps and $\sum_{s_{i+1}}\sum_{suf(s_i)}P_i(suf(s_{i+1}))\leq N$ since $\sum_{suf(s_i)}P_i(suf(s_{i+1}))=1$ and the number of choice of $s_{i+1}$ after $s_i$ is bounded by $N$, which can be seen from the definition of product MDP. Therefore we have proven that (\ref{eqn:desired}) holds and thus the lemma is proved.
\end{proof}

\begin{lemma}\label{lemma:optimal}
	If $\bar{\mathcal{M}}$ is an $\frac{\epsilon}{NTR_{max}D^2}-$approximation of $\mathcal{M}$, and $\epsilon>0$, then for any LTL specification $\phi$, for T-cycle optimal policy $f$ and $g$ defined in $\mathcal{C}\in\mathcal{C}(
	\mathcal{P})$ for $\bar{\mathcal{P}}$ and $\mathcal{P}$, we have that $|J^{T,f}_{\mathcal{P}}(s,q)-J^{T,g}_{\mathcal{P}}(s,q)|\leq2\epsilon$, where $(s,q)=ent(\mathcal{C})$ and $J^T$ is the average cost in $T$ cycles.    
\end{lemma} 
\begin{proof}
	This is the direct result from Lemma \ref{lemma:main} and the fact that $J^{T,f}_{\bar{\mathcal{P}}}(s,q)\leq J^{T,g}_{\bar{\mathcal{P}}}(s,q)$.
\end{proof}

Lemma \ref{lemma:optimal} is analogous to Lemma 2 in \cite{Fu-RSS-14} and intuitively bounds the performance of the optimal $T$ cycle policy in the estimated model when applied to the ground truth model. The finite $T$ cycle is chosen such that the average cost over $T$ cycle for the optimal policy $g$ satisfies the following definition.
\begin{figure}[!t]
	\removelatexerror
	\begin{algorithm}[H]
		\SetKwData{Left}{left}\SetKwData{This}{this}\SetKwData{Up}{up}
		\SetKwFunction{Union}{Union}\SetKwFunction{FindCompress}{FindCompress}
		\SetKwInOut{Input}{input}\SetKwInOut{Output}{output}
		\Input{  The state and action sets $S$ and $A$, the labeling function $L$ and the underlying transition system $TS_\mathcal{M}$ for the MDP $\mathcal{M}$.  The specification DRA $R_\phi$, $\epsilon,\delta$, (estimated) mixing cycle $T$.}
		\Output{policy $f:S\times Q\rightarrow A$}
		\BlankLine
		%\emph{special treatment of the first line}\;
		\nl Obtain the AMEC set $\mathcal{C}(\mathcal{P})$ with only the transition probabilities unknown and $f_{\rightarrow\mathcal{C}}$ for each $\mathcal{C}\in\mathcal{C}(\mathcal{P})$ from the product $TS_{\mathcal{M}}\times R_\phi$;\
		
		%			\nl $j\leftarrow 0$, $\mathcal{K}^0_i=$DFA $\mathcal{G}_i$ such that $\mathcal{L}(\mathcal{G}_i)=\Sigma_i^*$,$\widetilde{\mathcal{M}_i^0}\leftarrow\mathcal{M}_i$;\		
		
		%		 \For{$i\in[1,N]$}{$\mathcal{K}^0_i=$DFA $\mathcal{G}$ such that $\mathcal{L}(\mathcal{G})=A_i^*$\tcp*[r]{initialization }
		%			~\nl $\widetilde{\mathcal{M}_i^0}\leftarrow\mathcal{M}_i$;	}
		
		\nl	\For{$\mathcal{C}\in\mathcal{C}(\mathcal{P})$}{
			\nl $(s,q)=(s_0,q_0)$\;
			\nl apply $f_{\rightarrow\mathcal{C}}$ until reach $ent(\mathcal{C})$\;
			\nl $\bar{\mathcal{C}}$ $\leftarrow$ ModelLearning $(\mathcal{C})$\;
			\nl $f_{\circlearrowright\bar{\mathcal{C}}}\leftarrow$ optimal $T$-cycle policy\;
			\nl $J_{\bar{\mathcal{C}}}=J_{\bar{\mathcal{C}}}^{f,T}(ent(\bar{\mathcal{C}}))$\;	}

		\nl $\mathcal{C}^*=\arg\min_{\mathcal{\bar{C}}}{J_{\bar{\mathcal{C}}}}$\;
		\nl $f^*_{\rightarrow} = f_{\rightarrow{\mathcal{C}^*}}$\;
		\nl $f^*_{\circlearrowright} = f_{\circlearrowright{\mathcal{C}^*}}$\;
		\nl \Return {$f =[f^*_{\rightarrow},f^*_{\circlearrowright} ] $}\;
		
		~\caption{ModelLearningAndPolicyFinding}\label{alg:learning}			
	\end{algorithm}
\end{figure}
\begin{definition}
	Given the optimal policy $g$ defined in $\mathcal{C}$ over infinite time horizon, the $\epsilon$-mixing cycle  is defined to be the smallest positive integer $T_\mathcal{C}$ such that
	$$
	J^{g,T_\mathcal{C}}_{\mathcal{P},\mathcal{C}}(s,q)-J^{g}_{\mathcal{P},\mathcal{C}}(s,q)<\epsilon
	$$ 
	for $(s,q)=ent(\mathcal{C})$ where $J^{g,T_\mathcal{C}}_{\mathcal{P},\mathcal{C}}(s,q)
	$ is the average cost per cycle under policy $g$ from $(s,q)$ runs for $T$ cycles
	and $J^{g}_{\mathcal{P},\mathcal{C}}(s,q)$ is the optimal ACPC. The $\epsilon$-mixing cycle $T$ for the MDP is defined to be $T=\max_{\mathcal{C}\in\mathcal{C}(\mathcal{P})}\{T_\mathcal{C}\}$.
\end{definition}

From the above discussion, it can be seen that it is essential to obtain an estimated MDP that is statistically close to the ground truth one. To this regard, we need to estimate the transition probabilities of the MDP. We follow the results in \cite{Fu-RSS-14}, when the number of sampling is large enough, the maximum likelihood estimator of the transition probability $P(s,a,s')$ can be seen as a random variable of normal distribution with mean $\mu = \frac{\text{count}(s,a,s')}{\text{count}(s,a)}$ and variance $\sigma^2 = \frac{\text{count}(s,a,s')(\text{count}(s,a)-\text{count}(s,a,s'))}{\text{count}(s,a)^2(\text{count}(s,a)+1)}$, where $\text{count}(s,a)$ is the total number of times that action $a$ is executed on $s$ and $\text{count}(s,a,s')$ is the total number of times that the transition $s,a,s'$ is observed.

Then we use the idea of \emph{known states} \cite{brafman2002r,Fu-RSS-14,kearns2002near} as defined below. 

\begin{definition}\label{def:known}
	Given an MDP $\mathcal{M}$ and LTL specification $\phi$, a probabilistic transition $((s,q),a,(s',q'))$ from $\mathcal{P}=\mathcal{M}\times\mathcal{R}_\phi$ is known if with probability at least $1-\delta$, for any $(s,q')$ that is reachable by executing $a$, $\sigma^2 k\leq\frac{\epsilon}{NTR_{max}D^2}$ where $\sigma^2$ is the variance of the transition probability estimator, $k$ is the critical value for $1-\delta$ confidence interval, $\epsilon>0$ is a constant, $N=|S|$, $T$ is the $\epsilon$-mixing cycle and $L$ is the longest step required to finish one cycle. A state $(q,s)$ is known if any action $a$ defined on it, $((s,q),a,(s',q'))$ is known for any $(s',q')$. 
\end{definition}

From Definition \ref{def:known}, if all the states of the MDP are known, then with probability no less than $1-\delta$, it is an $\frac{\epsilon}{NTR_{max}D^2}$-approximation of the true underling MDP. In fact, since we are only interested in the behavior inside the AMECs, it is sufficient that all the states in the AMECs are known. Then we have the following algorithm and theorem to shown that we can learn such model efficiently.

\begin{theorem}
	Given an MDP model $\mathcal{M}$ and LTL specification $\phi$ in the form of (\ref{eqn:LTL}) and  satisfy the Assumption \ref{assump:structure}, \ref{assump:loop}, \ref{assump:AMEC} and \ref{assump:LTL}. Given $0<\delta<1$, $\epsilon>0$, $T$ is the $\epsilon$-mixing cycle. With probability no less than $1-\delta$, Algorithm \ref{alg:learning} will learn an estimated MDP model with all states in AMECs known within number of steps polynomial in $|S|,|A|,T,\frac{1}{\epsilon},\frac{1}{\delta}$ and linear in the number of accepting pairs. The $T$-cycle optimal policy $f :S\times Q\rightarrow A
	$ found on this estimated model will satisfy 
	$$
	J^{f,T}_{\mathcal{P}}(s,q)-J^{*}_{\mathcal{P}}(s,q)< 3\epsilon 
	$$ for $s=ent(\mathcal{C})$ and $\mathcal{C}$ incurs the minimum of $T$-cycle average cost.
\end{theorem}

\begin{proof}
	The upper bound on the number of visits to make a state known is polynomial in $|A|,T,\frac{1}{\epsilon},\frac{1}{\delta}$ from Chernoff bound. Before all states in one $c\in\mathcal{C}$ are known, the exploration policy makes sure that every state in $c$ will be sufficiently visited because $\mathcal{C}$ is a communicating MDP. Once all the $\mathcal{C}\in\mathcal{C}(\mathcal{P})$ are known which takes the number of steps polynomial in $|S|,|A|,T,\frac{1}{\epsilon},\frac{1}{\delta}$ and linear in the number of accepting pairs, for every $\mathcal{C}$, the optimal $T$ cycle policy $f_\mathcal{C}$  satisfies 
	$$
	|J^{f_\mathcal{C},T}_{\mathcal{P},\mathcal{C}}(s,q)-J^{g,T}_{\mathcal{P},\mathcal{C}}(s,q)|\leq2\epsilon	
	$$ from Lemma \ref{lemma:optimal}. From the definition of $\epsilon-$mixing time, we know that $J^{g,T}_{\mathcal{P},\mathcal{C}}(s,q)-J^*_{\mathcal{P},\mathcal{C}}(s,q)<\epsilon$. Therefore, 
	$$
	J^{f_\mathcal{C},T}_{\mathcal{P},\mathcal{C}}(s,q) - J^*_{\mathcal{P},\mathcal{C}}(s,q)<3\epsilon
	$$
	Since we select the AMEC with the smallest average cost, that is $J^{f,T}_{\mathcal{P}}(s,q)=\min_{\mathcal{C}\in\mathcal{C}(\mathcal{P})}{J^{f_\mathcal{C},T}_{\mathcal{P},\mathcal{C}}(s,q)}$, therefore $J^{f,T}_{\mathcal{P},\mathcal{C}}(s,q) - J^*_{\mathcal{P},\mathcal{C}}(s,q)<3\epsilon$ for all $\mathcal{C}\in\mathcal{C}(\mathcal{P})$, then we have that 
	$$
	J^{f,T}_{\mathcal{P}}(s,q)-J^{*}_{\mathcal{P}}(s,q)< 3\epsilon 
	$$	
\end{proof}

Note that $f_{\rightarrow\mathcal{C}}$ can be computed using standard value iteration \cite{rutten2004mathematical}. The sub-algorithm ModelLearning($\mathcal{C}$) in Algorithm \ref{alg:learning} can be flexible as long as it can explore all the state transitions in $\mathcal{C}$ sufficiently well. For example, for each state $(s,q)\in\mathcal{C}$, we could select the action in a round-robin fashion to make sure every state and action pair is sampled. In Section \ref{sec:simulation}, we solve the exploration and exploitation by following the  strategy  that for any unknown states, we select the action in a round-robin fashion while for states that are already known, we follow the strategy based on the current model that maximizes the probability to reach unknown states which can be efficiently computed by standard value iteration \cite{bertsekas1995dynamic}. As to obtain the optimal stationary $T$-cycle policy, we utilize the policy evaluation method in \cite{ding2014optimal} and find the policy that incurs the optimized $T-$cycle ACPC. 

\section{Example}\label{sec:simulation}
We illustrate the validity of our framework in a case study of a human robot collaboration in an  assembly scenario. The simulation is run in Matlab on a laptop with Intel i7 processor with 2.6GHz  and 16GB of memory. 

\begin{table*}[t]
	\centering
	\begin{tabular}{|l|l|l|l|l|l|l|l|l|l|l|l|l|}
		\hline
		$f_0^r$ & $f_1^r$  & $f_0^h$ & $f_1^h$ & $f_2^h$& $f_{repair}$ & $t_0^r$ & $t_1^r$& $t_{1,0}^r$ & $t_{1,1}^r$\\ 
		\hline
		0.5 & 0.4 &  0.5 & 0.4 & 0.45 & 0.4 & 0.5 & 0.4 & 0.3 & 0.4 \\
		\hline
		
	\end{tabular}
	\caption{Transition probabilities}
	\label{table:transition}
\end{table*}

For the underlying true model to be learned, the task model is as shown in Figure \ref{fig:spec}. The robot model is as shown in Figure \ref{fig:robot} with $p_0=0.6,p_1=0.65$. The fatigue model and trust model are as shown in Figure \ref{fig:fatigue} and Figure \ref{fig:trust}. Their corresponding transition probabilities are as shown in Table \ref{table:transition}. For the fatigue model, it can be seen that for each human action $a_i^h$ and $repair$, there are probabilities $f_i^h$ and $f_{repair}$ for the human to stay at the same fatigue level and $1-f_i^h,1-f_{repair}$ to move to a higher fatigue level if possible. If the human is already at the highest fatigue level, any action involves human wouldn't make any change. If the robot is selected to perform actions $a_i^r$, then the human would be idle. There is $f^i_r$ probability for the human to stay at the same fatigue level and $1-f^i_r$ probability to decrease to the lower fatigue level. When the human is at the lowest fatigue level, any idleness wouldn't change the fatigue level.

As shown in the human trust model in Figure \ref{fig:trust}, human action would not affect human trust to the robots. Any robot action would have certain probability for the human to stay at the same trust level,  and certain probability to transfer to adjacent trust level when applicable. The transition probabilities imply the robot's performance  as there is higher chance to increase human trust if the robot performs it well. 

\begin{figure}[h]
	\centering	
	\begin{tikzpicture}[shorten >=1pt,node distance=4cm,on grid,auto, thick,scale=.75, every node/.style={transform shape}]
	\node[state,initial] (f_0)   {$f_0$};
	\node[state] (f_1) [right=of f_0] {$f_1$};
	\node[state] (f_2) [right=of f_1] {$f_2$};
	
	%\node[state] (s_4) [below=of s_3] {};
	%	\node[state] (s_5) [left=of s_4] {};
	%	\node[state] (s_6) [left=of s_5] {};
	\path[->]
	(f_0) edge [pos=0.5, loop, above=0.1] node [align=center] {$a_i^r,1$\\$a_i^h,f_i^h$\\$repair,f_{repair}$} (f_0)
	(f_0) edge [pos=0.5, bend left, above=0.5] node [align=center] {$a_i^h,1-f_i^h$\\$repair,1-f_{repair}$} (f_1)
	(f_1) edge [pos=0.5, loop, above=0.1] node [align=center] {$a_i^r,f_i^r$\\$a_i^h,f_i^h$\\$repair,f_{repair}$} (f_1)
	(f_1) edge [pos=0.5, bend left, below=0.5] node [align=center] {$a_i^r,1-f_i^r$} (f_0)
	(f_1) edge [pos=0.5, bend left, above=0.5] node [align=center] {$a_i^h,1-f_i^h$\\$repair,1-f_{repair}$} (f_2)
	(f_2) edge [pos=0.5, loop, above=0.1] node [align=center] {$a_i^r,f_i^r$\\$a_i^h,1$\\$repair,1$} (f_2)
	(f_2) edge [pos=0.5, bend left, below=0.5] node [align=center] {$a_i^r,1-f_i^r$} (f_1)
	;
	
	%(s_3) edge node [pos=0.5, sloped, above]{$D_1 open$} (s_4)	
	%	(s_4) edge node [pos=0.5, sloped, above]{$R_2 to 1$} (s_5)
	%	(s_5) edge node [pos=0.5, sloped, above]{$R_2 in 1$} (s_6)
	%	(s_3) edge node [pos=0.5, sloped, above]{$r_1$} (s_0);	
	\end{tikzpicture}
	\caption{Human fatigue model}
	\label{fig:fatigue}
\end{figure}
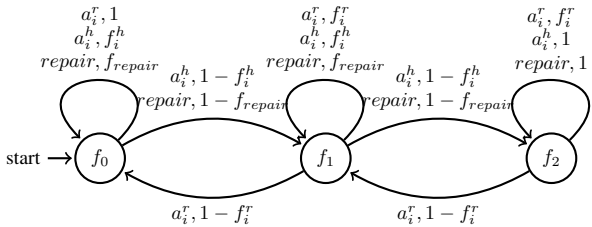

\begin{figure}[h]
	\centering	
	\begin{tikzpicture}[shorten >=1pt,node distance=4cm,on grid,auto, thick,scale=0.75, every node/.style={transform shape}]
	\node[state,initial] (t_0)   {$t_0$};
	\node[state] (t_1) [right=of t_0] {$t_1$};
	\node[state] (t_2) [right=of t_1] {$t_2$};
	
	%\node[state] (s_4) [below=of s_3] {};
	%	\node[state] (s_5) [left=of s_4] {};
	%	\node[state] (s_6) [left=of s_5] {};
	\path[->]
	(t_0) edge [pos=0.5, loop, above=0.1] node [align=center] {$a_i^r,t_i^r$\\$repair,1$} (t_0)
	(t_0) edge [pos=0.5, bend left, above=0.5] node [align=center] {$a_i^r,1-t_i^r$} (t_1)
	(t_1) edge [pos=0.5, loop, above=0.1] node [align=center] {$a_i^r,t_{1,i}^r$} (t_1)
	(t_1) edge [pos=0.5, bend left, below=0.5] node [align=center] {$a_i^r,(1-t_{1,i}^r)/2$\\$repair,1$} (t_0)
	(t_1) edge [pos=0.5, bend left, above=0.5] node [align=center] {$a_i^r,(1-t_{1,i}^r)/2$} (t_2)
	(t_2) edge [pos=0.5, loop, above=0.1] node [align=center] {$a_i^r,t_i^r$} (t_2)
	(t_2) edge [pos=0.5, bend left, below=0.5] node [align=center] {$a_i^r,1-t_i^r$\\$repair,1$} (t_1)
	;
	
	%(s_3) edge node [pos=0.5, sloped, above]{$D_1 open$} (s_4)	
	%	(s_4) edge node [pos=0.5, sloped, above]{$R_2 to 1$} (s_5)
	%	(s_5) edge node [pos=0.5, sloped, above]{$R_2 in 1$} (s_6)
	%	(s_3) edge node [pos=0.5, sloped, above]{$r_1$} (s_0);	
	\end{tikzpicture}
	\caption{Human trust model}
	\label{fig:trust}
\end{figure}
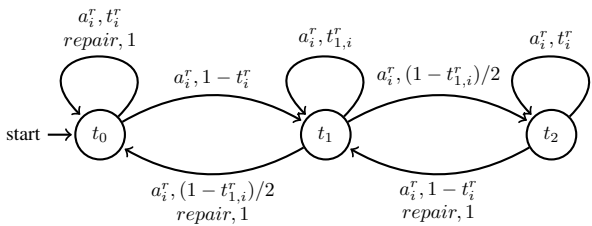

The four models are created using PRISM model checker \cite{kwiatkowska2011prism} which also produces the composed system model $\mathcal{M}$. As for the cost, we assume it is unknown to the learner initially and is learned later. For the task model, all actions in the task model would incur a cost of $0.7$. For the robot model, robot actions $a_0^r$ and $a_1^r$ incur $0.003$ and $0.07$ respectively, the cost of repair is $0.07$. The cost for human fatigue and trust models are as shown in Table \ref{table:cost} where the empty entries mean not applicable. Essentially we penalize the following cases by assigning a larger cost 1) human is idle and the fatigue level is low; 2) human is busy and the fatigue level is high; 3) human trust is low; 4) lost of trust due to robot failure. On the other hand, we reward the following cases 1) human is busy when fatigue level is low; 2) human is idle when fatigue level is high; 3) human is at the medium fatigue level 4) When human trust is high. 
\begin{table}[t]
	\centering
	\begin{tabular}{|l|l|l|l|l|l|l|}
		\hline
		& $f_0$ & $f_1$ & $f_2$ & $t_0$ & $t_1$ & $t_2$\\ 
		\hline
		$a_i^r$ & 0.3 & 0.1  & 0.03 & 0.3 & 0.17 & 0.03  \\
		\hline
		$a_i^h$ & 0.03 & 0.1 & 0.3 &  & &   \\
		\hline
		$repair$& 0.03  & 0.1 & 0.3   & 0.17 & 0.17 & 0.5  \\
		
		\hline
		
	\end{tabular}
	\caption{Costs associated to each state and action }
	\label{table:cost}
\end{table}

The action set for composed $\mathcal{M}$ is $\{a_0^r,a_1^r,a_0^h,$ $a_1^h,a_2^h,repair\}$. The total number of states in this example is $54$. The states refer to the tuple $(s^w,s^r,s^t,s^f)$ where each element in the tuple is from task automaton, robot model, human fatigue model and human trust model respectively.

We label $S_{normal}=\{*,0,*,*\}$ as $normal$, for all the states when the robot is in the normal state. Additionally, we label  $S_{faulty}=\{*,1,*,*\}$  as $faulty$ for the all the states when the robot is in the faulty state. Our LTL specification is as the following
\begin{equation}\label{eqn:LTL1}
	\phi_1=\square\diamondsuit\pi\wedge  \square (faulty \rightarrow  X normal)
\end{equation}
In plain words, it requires the composed MDP $\mathcal{M}$ to visit the finish state in $S_\pi$ infinitely often and whenever there is fault, it will be fixed in the next step, since leaving robot in faulty state for too long may permanently damage it. The corresponding DRA has $9$ states and one acceptance pair. Therefore the product model has $486$ states.

For this example, we set $\epsilon=0.35,T=10, R_{max}=1$ and $90\%$ confidence level. Even if the true model is unknown before learning, since the underline transition system is known, we can find that the maximum step to finish a cycle is $D=5$ and calculate the AMEC. There is one AMEC with $75$ states in the product model which includes the initial state $(q_0,s_0)$, therefore there is no need to find $f_{\rightarrow\mathcal{C}}$. The model learning took  $206$ seconds and $812872$ cycles. Then we obtain the $T$-cycle optimal policy by evaluating the $T$-cycle ACPC with formula (21) in from \cite{ding2014optimal} for each possible deterministic and memoryless policy. The optimal $T-$cycle ACPC we obtain is $1.134$ while the optimal ACPC with infinite time horizon is $1.128$. The $T-$cycle optimal policy we found is almost identical to the one for unbounded case with only one different action choice. It can be seen that the performance bound we have is quite conservative due to the loose upper bound.

	\section{Conclusion}
	\label{sec: conclusion}
	In this paper we proposed a learning based dynamic task assignment framework for a class of human robot interactive systems that human and robot work closely to finish a given mission repetitively. Due to the uncertainties in robot dynamics and the evolution of human cognitive and physiological states, we model the robot status, human trust and fatigue as MDPs respectively. A specification is given to ensure that the mission can be finished infinitely often with the optimal probability and additional constraints can be added as necessary. The abstract model of the human robot system is a composed MDP from the task model, robot model, human fatigue and trust models. The task assignment problem is then converted to an optimal controller synthesis problem. When the system model is unknown, we proposed a PAC-learning inspired algorithm to efficiently learn the model and then find the optimal policy. An example has been explored to show the validity of our framework. Future work includes extending the problem to more general settings, such as two-player stochastic games to consider the case where the human action is not controllable and partial observability where the human states such as trust and fatigue cannot be directly observed. 
	
	%In this paper we assume that the robot model, human fatigue and trust models are given and we have direct observation of the states, however in practice the model itself contains uncertainties and we may not be able to directly observe the states such as fatigue and trust levels. Therefore our future work would be to extend our framework to take these practical issues into consideration, such as generalize both the model and the algorithm to Partially Observable MDPs (POMDPs), uncertainties in the transition probabilities  and consider more general LTL formulas.
	%	\section*{Acknowledgment}
	
	\bibliographystyle{IEEEtran}
	\bibliography{ref}

\begin{thebibliography}{10}
\providecommand{\url}[1]{#1}
\csname url@samestyle\endcsname
\providecommand{\newblock}{\relax}
\providecommand{\bibinfo}[2]{#2}
\providecommand{\BIBentrySTDinterwordspacing}{\spaceskip=0pt\relax}
\providecommand{\BIBentryALTinterwordstretchfactor}{4}
\providecommand{\BIBentryALTinterwordspacing}{\spaceskip=\fontdimen2\font plus
\BIBentryALTinterwordstretchfactor\fontdimen3\font minus
  \fontdimen4\font\relax}
\providecommand{\BIBforeignlanguage}[2]{{%
\expandafter\ifx\csname l@#1\endcsname\relax
\typeout{** WARNING: IEEEtran.bst: No hyphenation pattern has been}%
\typeout{** loaded for the language `#1'. Using the pattern for}%
\typeout{** the default language instead.}%
\else
\language=\csname l@#1\endcsname
\fi
#2}}
\providecommand{\BIBdecl}{\relax}
\BIBdecl

\bibitem{sadrfaridpour2016modeling}
B.~Sadrfaridpour, H.~Saeidi, J.~Burke, K.~Madathil, and Y.~Wang, ``Modeling and
  control of trust in human-robot collaborative manufacturing,'' in
  \emph{Robust Intelligence and Trust in Autonomous Systems}.\hskip 1em plus
  0.5em minus 0.4em\relax Springer, 2016, pp. 115--141.

\bibitem{chen2014optimal}
F.~Chen, K.~Sekiyama, F.~Cannella, and T.~Fukuda, ``Optimal subtask allocation
  for human and robot collaboration within hybrid assembly system,'' \emph{IEEE
  Transactions on Automation Science and Engineering}, vol.~11, no.~4, pp.
  1065--1075, 2014.

\bibitem{broadbent2009acceptance}
E.~Broadbent, R.~Stafford, and B.~MacDonald, ``Acceptance of healthcare robots
  for the older population: review and future directions,'' \emph{International
  journal of social robotics}, vol.~1, no.~4, pp. 319--330, 2009.

\bibitem{okamura2010medical}
A.~M. Okamura, M.~J. Mataric, and H.~I. Christensen, ``Medical and health-care
  robotics,'' \emph{IEEE Robotics \& Automation Magazine}, vol.~17, no.~3, pp.
  26--37, 2010.

\bibitem{lin2015experiments}
C.-W. Lin, M.-H. Khong, and Y.-C. Liu, ``Experiments on human-in-the-loop
  coordination for multirobot system with task abstraction,'' \emph{IEEE
  Transactions on Automation Science and Engineering}, vol.~12, no.~3, pp.
  981--989, 2015.

\bibitem{seshia2015formal}
S.~A. Seshia, D.~Sadigh, and S.~S. Sastry, ``Formal methods for semi-autonomous
  driving,'' in \emph{Proceedings of the 52nd Annual Design Automation
  Conference}.\hskip 1em plus 0.5em minus 0.4em\relax ACM, 2015, p. 148.

\bibitem{zanchettin2016safety}
A.~M. Zanchettin, N.~M. Ceriani, P.~Rocco, H.~Ding, and B.~Matthias, ``Safety
  in human-robot collaborative manufacturing environments: Metrics and
  control,'' \emph{IEEE Transactions on Automation Science and Engineering},
  vol.~13, no.~2, pp. 882--893, 2016.

\bibitem{baier2008principles}
C.~Baier and J.~Katoen, \emph{Principles of Model Checking}.\hskip 1em plus
  0.5em minus 0.4em\relax MIT Press, 2008.

\bibitem{bauer2008human}
A.~Bauer, D.~Wollherr, and M.~Buss, ``Human--robot collaboration: a survey,''
  \emph{International Journal of Humanoid Robotics}, vol.~5, no.~01, pp.
  47--66, 2008.

\bibitem{hancock2011meta}
P.~A. Hancock, D.~R. Billings, K.~E. Schaefer, J.~Y. Chen, E.~J. De~Visser, and
  R.~Parasuraman, ``A meta-analysis of factors affecting trust in human-robot
  interaction,'' \emph{Human Factors: The Journal of the Human Factors and
  Ergonomics Society}, vol.~53, no.~5, pp. 517--527, 2011.

\bibitem{desai2012effects}
M.~Desai, M.~Medvedev, M.~V{\'a}zquez, S.~McSheehy, S.~Gadea-Omelchenko,
  C.~Bruggeman, A.~Steinfeld, and H.~Yanco, ``Effects of changing reliability
  on trust of robot systems,'' in \emph{Human-Robot Interaction (HRI), 2012 7th
  ACM/IEEE International Conference on}.\hskip 1em plus 0.5em minus 0.4em\relax
  IEEE, 2012, pp. 73--80.

\bibitem{robinette2015effect}
P.~Robinette, A.~R. Wagner, and A.~M. Howard, ``The effect of robot performance
  on human--robot trust in time--critical situations,'' 2015.

\bibitem{lee1992trust}
J.~Lee and N.~Moray, ``Trust, control strategies and allocation of function in
  human-machine systems,'' \emph{Ergonomics}, vol.~35, no.~10, pp. 1243--1270,
  1992.

\bibitem{desai2013impact}
M.~Desai, P.~Kaniarasu, M.~Medvedev, A.~Steinfeld, and H.~Yanco, ``Impact of
  robot failures and feedback on real-time trust,'' in \emph{Proceedings of the
  8th ACM/IEEE international conference on Human-robot interaction}.\hskip 1em
  plus 0.5em minus 0.4em\relax IEEE Press, 2013, pp. 251--258.

\bibitem{wang2015mutual}
X.~Wang, Z.~Shi, F.~Zhang, and Y.~Wang, ``Mutual trust based scheduling for
  (semi) autonomous multi-agent systems,'' in \emph{2015 American Control
  Conference (ACC)}.\hskip 1em plus 0.5em minus 0.4em\relax IEEE, 2015, pp.
  459--464.

\bibitem{Fu-RSS-14}
J.~Fu and U.~Topcu, ``Probably approximately correct mdp learning and control
  with temporal logic constraints,'' in \emph{Proceedings of Robotics: Science
  and Systems}, Berkeley, USA, July 2014.

\bibitem{wolff2012optimal}
E.~M. Wolff, U.~Topcu, and R.~M. Murray, ``Optimal control with weighted
  average costs and temporal logic specifications,'' in \emph{Robotics: Science
  and Systems}, 2012.

\bibitem{ding2014optimal}
X.~Ding, S.~L. Smith, C.~Belta, and D.~Rus, ``Optimal control of markov
  decision processes with linear temporal logic constraints,'' \emph{Automatic
  Control, IEEE Transactions on}, vol.~59, no.~5, pp. 1244--1257, 2014.

\bibitem{svorenova2013optimal}
M.~Svorenova, I.~Cerna, and C.~Belta, ``Optimal control of mdps with temporal
  logic constraints,'' in \emph{Decision and Control (CDC), 2013 IEEE 52nd
  Annual Conference on}.\hskip 1em plus 0.5em minus 0.4em\relax IEEE, 2013, pp.
  3938--3943.

\bibitem{sadigh2014learning}
D.~Sadigh, E.~S. Kim, S.~Coogan, S.~S. Sastry, and S.~A. Seshia, ``A learning
  based approach to control synthesis of markov decision processes for linear
  temporal logic specifications,'' in \emph{Decision and Control (CDC), 2014
  IEEE 53rd Annual Conference on}.\hskip 1em plus 0.5em minus 0.4em\relax IEEE,
  2014, pp. 1091--1096.

\bibitem{ji2006probabilistic}
Q.~Ji, P.~Lan, and C.~Looney, ``A probabilistic framework for modeling and
  real-time monitoring human fatigue,'' \emph{IEEE Transactions on systems,
  man, and cybernetics-Part A: Systems and humans}, vol.~36, no.~5, pp.
  862--875, 2006.

\bibitem{fujieSharedAutonomy}
J.~Fu and U.~Topcu, ``Synthesis of shared autonomy policies with temporal logic
  specifications,'' \emph{IEEE Transactions on Automation Science and
  Engineering}, vol.~13, no.~1, pp. 7--17, Jan 2016.

\bibitem{ACC2017}
\BIBentryALTinterwordspacing
B.~Wu, B.~Hu, and H.~Lin, ``Toward efficient manufaturing systems: a trust
  based human robot collaboration,'' in \emph{American Control Conference
  (ACC), 2015}.\hskip 1em plus 0.5em minus 0.4em\relax IEEE, 2017, pp.
  1536--1541. [Online]. Available:
  \url{http://www3.nd.edu/~bwu3/doc/ACC2017.pdf}
\BIBentrySTDinterwordspacing

\bibitem{puterman2014markov}
M.~L. Puterman, \emph{Markov decision processes: discrete stochastic dynamic
  programming}.\hskip 1em plus 0.5em minus 0.4em\relax John Wiley \& Sons,
  2014.

\bibitem{muir1990operators}
B.~M. Muir, \emph{Operators' trust in and use of automatic controllers in a
  supervisory process control task}.\hskip 1em plus 0.5em minus 0.4em\relax
  University of Toronto, 1990.

\bibitem{abbeel2004apprenticeship}
P.~Abbeel and A.~Y. Ng, ``Apprenticeship learning via inverse reinforcement
  learning,'' in \emph{Proceedings of the twenty-first international conference
  on Machine learning}.\hskip 1em plus 0.5em minus 0.4em\relax ACM, 2004, p.~1.

\bibitem{brafman2002r}
R.~I. Brafman and M.~Tennenholtz, ``R-max-a general polynomial time algorithm
  for near-optimal reinforcement learning,'' \emph{Journal of Machine Learning
  Research}, vol.~3, no. Oct, pp. 213--231, 2002.

\bibitem{kearns2002near}
M.~Kearns and S.~Singh, ``Near-optimal reinforcement learning in polynomial
  time,'' \emph{Machine Learning}, vol.~49, no. 2-3, pp. 209--232, 2002.

\bibitem{rutten2004mathematical}
J.~J. Rutten, M.~Kwiatkowska, G.~Norman, and D.~Parker, \emph{Mathematical
  techniques for analyzing concurrent and probabilistic systems}.\hskip 1em
  plus 0.5em minus 0.4em\relax American Mathematical Soc., 2004.

\bibitem{bertsekas1995dynamic}
D.~P. Bertsekas, \emph{Dynamic Programming and Optimal Control}, 4th~ed.\hskip
  1em plus 0.5em minus 0.4em\relax Athena Scientific, 2012, vol.~2.

\bibitem{kwiatkowska2011prism}
M.~Kwiatkowska, G.~Norman, and D.~Parker, ``Prism 4.0: Verification of
  probabilistic real-time systems,'' in \emph{International Conference on
  Computer Aided Verification}.\hskip 1em plus 0.5em minus 0.4em\relax
  Springer, 2011, pp. 585--591.

\end{thebibliography}
\end{document}